\documentclass{article}

\usepackage[left=1in, right=1in, top=1in, bottom=1in]{geometry}

\usepackage[square,sort,comma,numbers]{natbib}

\usepackage{microtype}
\usepackage{graphicx}
\usepackage{subfigure}
\usepackage{array}
\usepackage{booktabs} 
\usepackage{amsfonts}
\usepackage{amssymb,amsmath,amsthm,bm}
\newtheorem{remark}{Remark}
\newtheorem{theorem}{Theorem}
\newtheorem{definition}{Definition}
\newtheorem{lemma}{Lemma}

\usepackage{soul}
\usepackage{lmodern}
\usepackage{booktabs}
\usepackage{threeparttable}
\usepackage{mathtools}
\usepackage{tablefootnote}

\usepackage{cuted}
\usepackage{lipsum, color}

\usepackage{verbatim}

\DeclarePairedDelimiter\floor{\lfloor}{\rfloor}

\usepackage{hyperref}

\usepackage{listings}
\usepackage{color}

\usepackage{color}

\definecolor{mygreen}{rgb}{0,0.6,0}
\definecolor{mygray}{rgb}{0.5,0.5,0.5}
\definecolor{mymauve}{rgb}{0.58,0,0.82}

\usepackage{algorithm}
\usepackage[noend]{algpseudocode}

\def\va{{\bm{a}}}
\def\vb{{\bm{b}}}

\def\vv{{\bm{v}}}
\def\vw{{\bm{w}}}
\def\vx{{\bm{x}}}

\def \RR {{\mathbb{R}}}

\def \bx {{\bm x}}

\def \Ab {{\mathbf{A}}}

\usepackage{soul}

\title{Scheduled Restart Momentum for Accelerated Stochastic Gradient Descent}

\author{
Bao Wang$^*$ \\
  Scientific Computing and Imaging (SCI) Institute\\
  University of Utah, Salt Lake City, UT, USA\\
  \and
  Tan M. Nguyen\footnote{Co-first author. Please correspond to: wangbaonj@gmail.com or mn15@rice.edu} \\
  Department of ECE\\ 
  Rice University, Houston, USA\\
  \and
  Tao Sun \\
  College of Computer\\
  National University of Defense Technology\\
  Changsha, China\\
  \and
  Andrea L. Bertozzi \\
  Department of Mathematics\\
  University of California, Los Angeles\\
  \and
  Richard G. Baraniuk$^\dag$\\
  Department of ECE\\ 
  Rice University, Houston, USA\\
  \and
  Stanley J. Osher\footnote{Co-last author} \\
  Department of Mathematics\\
  University of California, Los Angeles\\
}

\begin{document}

\maketitle

\begin{abstract}
Stochastic gradient descent (SGD) with constant momentum and its variants such as Adam are the optimization algorithms of choice for training deep neural networks (DNNs). Since DNN training is incredibly computationally expensive, there is great interest in speeding up the convergence. Nesterov accelerated gradient (NAG) improves the convergence rate of gradient descent (GD) for convex optimization using a specially designed momentum; however, it accumulates error when an inexact gradient is used (such as in SGD), slowing convergence at best and diverging at worst. In this paper, we propose \emph{Scheduled Restart SGD} (SRSGD), a new NAG-style scheme for training DNNs. SRSGD replaces the constant momentum in SGD by the increasing momentum in NAG but stabilizes the iterations by resetting the momentum to zero according to a schedule. Using a variety of models and benchmarks for image classification, we demonstrate that, in training DNNs, SRSGD significantly improves convergence and generalization; for instance in training ResNet200 for ImageNet classification, SRSGD achieves an error rate of 20.93\% vs.\ the benchmark of 22.13\%. These improvements become more significant as the network grows deeper. Furthermore, on both CIFAR and ImageNet, SRSGD reaches similar or even better error rates with significantly fewer training epochs compared to the SGD baseline.
\end{abstract}

\section{Introduction}
Training many machine learning (ML) models reduces to solving the following finite-sum optimization problem
{\begin{equation}
\label{eq:finite-sum}
\min_\vw f(\vw) := \min_\vw \frac{1}{N}\sum_{i=1}^N f_i(\vw),\ \ \vw \in \RR^d,
\end{equation}}
where $f_i(\vw) := \mathcal{L}(g(\bx_i, \vw), y_i)$ is the loss between the ground-truth label $y_i$ and the prediction by the model $g(\cdot, \vw)$, parametrized by $\vw$. This training loss is typically a cross-entropy loss for classification and a root mean square error for regression. Here $\{\bx_i, y_i\}_{i=1}^N$ are the training samples, and problem \eqref{eq:finite-sum} is known as {\em empirical risk minimization} (ERM). For many practical applications, $f(\vw)$ is highly non-convex, and $g(\cdot, \vw)$ is chosen among deep neural networks (DNNs) due to their preeminent performance across various tasks. These deep models are heavily overparametrized and require large amounts of training data. Thus, both $N$ and the dimension of $\vw$ can scale up to millions or even billions. These complications pose serious computational challenges.

One of the simplest algorithms to solve 
\eqref{eq:finite-sum} is gradient descent (GD), which updates $\vw$ according to:
{\begin{equation}
\label{eq:GD}
\vw^{k+1} = \vw^k - s_k\frac{1}{N}\sum_{i=1}^N\nabla f_i(\vw^k),
\end{equation}}
where $s_k > 0$ is the step size at the $k$-th iteration. Computing $\nabla f(\vw^k)$ on the entire training set is memory intensive and often cannot fit on devices with limited random access memory (RAM) such as graphics processing units (GPUs) typically used for deep learning (DL). In practice, we 
sample a 
subset of 
the training set, 
of size $m$ with $m\ll N$, 
to approximate $\nabla f(\vw^k)$ by the mini-batch gradient $1/m\sum_{j=1}^m \nabla f_{i_j}(\vw^k)$. This results in the stochastic gradient descent (SGD) update
{\begin{equation}
\label{eq:SGD}
\vw^{k+1} = \vw^k - s_k\frac{1}{m}\sum_{j=1}^m\nabla f_{i_j}(\vw^k).
\end{equation}}
SGD and its accelerated variants are among the most used optimization algorithms in ML practice \cite{bottou2018optimization}. These gradient-based algorithms have a number of benefits. Their convergence rate is usually independent of the dimension of the underlying problem \cite{bottou2018optimization}; their computational complexity is low and easy to parallelize, which makes them suitable to large scale and high dimensional problems \cite{zinkevich2010parallelized,zhang2015deep}. They have achieved, so far, the best performance in training 
 DNNs \cite{goodfellow2016deep}.

Nevertheless, GD and SGD have convergence issues, especially when the problem is ill-conditioned. There are two common approaches to accelerate GD:
adaptive step size \cite{duchi2011adaptive,hinton2012neural,zeiler2012adadelta} and momentum \cite{polyak1964some}. 
The integration of both adaptive step size and momentum with SGD leads to Adam \cite{kingma2014adam}, which is one of the most used optimizers for DNNs. Many recent developments have improved Adam \cite{reddi2019convergence,dozat2016incorporating,loshchilov2018fixing,liu2020on}. 
GD with constant momentum leverages previous step information to accelerate GD according to:
{\begin{eqnarray}
\label{eq:ConstanMomentum}
\begin{aligned}
  \vv^{k+1} &= \vw^k -  s_k\nabla f(\vw^k),\\
  \vw^{k+1} &= \vv^{k+1} + \mu (\vv^{k+1} - \vv^{k}),
\end{aligned}
\end{eqnarray}}
where $\mu > 0$ is a constant.
A similar acceleration can be achieved by the heavy-ball (HB) method \cite{polyak1964some}. Both momentum update in \eqref{eq:ConstanMomentum} and HB enjoy the same convergence rate of $O(1/k)$ as GD for convex smooth optimization. A breakthrough due to Nesterov \cite{nesterov1983method} replaces the constant momentum $\mu$ with $(k-1)/(k+2)$ (aka, Nesterov accelerated gradient (NAG) momentum), and it can accelerate the convergence rate to $O(1/k^2)$, which is optimal for convex and smooth loss functions \cite{nesterov1983method,su2014differential}. Jin et al. showed that NAG can also speed up escaping saddle point  \cite{jin2017accelerated}. In practice, NAG momentum and its variants such as Katyusha momentum \cite{allen2017katyusha} can also accelerate GD for  nonconvex optimization, especially when the underlying loss function is poorly conditioned \cite{goh2017momentum}.

However, Devolder et al. \cite{devolder2014first} has recently showed that NAG accumulates error when an inexact gradient
is used, thereby slowing convergence at best and diverging at worst. 
Until now, only constant momentum has been successfully used in training DNNs in practice \cite{sutskever2013importance}. Since NAG momentum has achieved a much better convergence rate than constant momentum methods with exact gradient oracle, in this paper we study the following question:

\emph{Can we leverage NAG momentum to accelerate SGD and to improve 
generalization in training DNNs?}

\paragraph{Contributions.}
We answer the above question by proposing the first algorithm that integrates scheduled restart (SR) NAG momentum with plain SGD. 
We name the resulting algorithm scheduled restart SGD (SRSGD). \emph{Theoretically, we present the 
error accumulation of Nesterov accelerated SGD (NASGD) and the convergence of SRSGD.} The major practical benefits of SRSGD are fourfold:
\begin{itemize}
    \item SRSGD can significantly speed up DNN training. For image classification, SRSGD can \emph{significantly reduce the number of training epochs while preserving or even improving the network's accuracy}. In particular, on CIFAR10/100, the number of training epochs can be reduced by half with SRSGD while on ImageNet the reduction in training epochs is also remarkable.
    \vspace{0.05in}
    \item DNNs trained by SRSGD \emph{generalize significantly better} than the current benchmark optimizers. 
     \vspace{0.05in}
    \emph{The improvement becomes more significant as the network grows deeper} as shown in Fig.~\ref{fig:error-vs-depth}. 
    \item SRSGD \emph{reduces overfitting in very deep networks} such as ResNet-200 for ImageNet classification, enabling the accuracy to keep increasing with depth.
    \vspace{0.05in}
    \item SRSGD is \emph{straightforward to implement} and only requires changes in a few lines of the SGD code. There is also no additional computational or memory overhead.
\end{itemize}
We focus on DL for image classification, in which SGD with constant momentum is the 
choice. 

\begin{figure}[t!]
\centering
\includegraphics[width=1.0\linewidth]{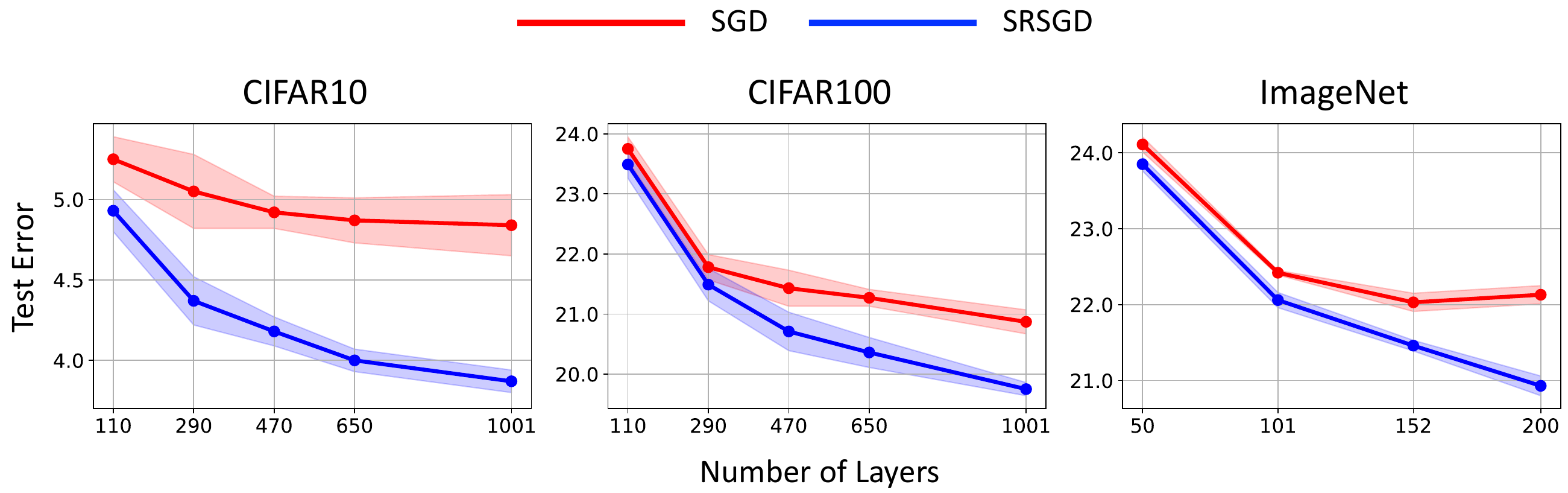}
\caption{Error vs. depth of ResNet models trained with SRSGD and the baseline SGD with constant momemtum. Advantage of SRSGD continues to grow with depth.}
\label{fig:error-vs-depth}
\end{figure}

\paragraph{Organization.}
In Section~\ref{section:momentum}, we review and discuss 
momentum 
for accelerating GD in convex smooth optimization. In Section~\ref{section:SRSGD}, we present scheduled restart NAG momentum to accelerate SGD, namely SRSGD algorithm and its theoretical guarantees. In Section~\ref{section:experiments}, we verify the efficacy of the proposed SRSGD in training DNNs for image classification on CIFAR and ImageNet. In Section~\ref{section:Empirical-Analysis}, we perform some empirical analysis of SRSGD. In Section~\ref{section:related_work}, we briefly review some 
representative works that utilize momentum to accelerate SGD and study the restart techniques in NAG. We end with concluding remarks. Technical proofs and some more experimental details and results, in particular training RNNs and GANs, are provided in the appendix.

\paragraph{Notation.}
We denote scalars and vectors by lower case and lower case bold face letters, respectively,
and matrices by upper case bold face letters. For a vector $\vx = (x_1, \cdots, x_d)\in \mathbb{R}^d$, we denote its $\ell_p$ norm ($p\geq 1$) 
by $\|\vx\|_p = {\left(\sum_{i=1}^d |x_i|^p\right)^{1/p}}$, the $\ell_\infty$ norm of $\vx$ by $\|\vx\|_\infty = \max_{i=1}^d|x_i|$. 
For a matrix $\Ab$, we used $\|\Ab\|_p$ to denote its induced norm by the vector $\ell_p$ norm. Given two sequences $\{a_n\}$ and $\{b_n\}$, we write $a_n=O(b_n)$ if there exists 
a positive constant s.t.
$0<C<+\infty$ such that 
$a_n \leq C b_n$. We denote the interval $a$ to $b$ (included) as $(a, b]$. 
For a function $f(\vw): \mathbb{R}^d \rightarrow \mathbb{R}$, we denote its gradient and Hessian as $\nabla f(\vw)$ and $\nabla^2 f(\vw)$, respectively.

\section{Review: Momentum in Gradient Descent}\label{section:momentum}
\subsection{Gradient Descent}
Perhaps the simplest algorithm to solve \eqref{eq:finite-sum} is GD \eqref{eq:GD}, which dates back to \cite{cauchy1847methode}. If the objective $f(\vw)$ is convex and $L$-smooth (i.e., $\|\nabla^2 f(\vw)\|_2\leq L$), then GD converges with rate $O(1/k)$ by letting $s_k \equiv 1/L$ (we use this $s_k$ in all the discussion below), which is independent of the dimension of $\vw$.

\subsection{Gradient Descent with Momentum -- Heavy Ball}
HB scheme \eqref{eq:HeavyBall} \cite{polyak1964some} accelerates GD by using the momentum $\vw^k - \vw^{k-1}$, which gives
{
\begin{equation}
\label{eq:HeavyBall}
\vw^{k+1} = \vw^k - s_k\nabla f(\vw^k) + \mu (\vw^k - \vw^{k-1}),
\end{equation}}
where $\mu > 0$ is a constant. Alternatively, we can accelerate GD by using the Nesterov  momentum (aka, lookahead momentum), which leads to the scheme in \eqref{eq:ConstanMomentum}. Both HB and 
\eqref{eq:ConstanMomentum} have the same convergence rate of $O(1/k)$ for solving convex smooth problems. Recently, several variants of \eqref{eq:ConstanMomentum} have been proposed for DL,  e.g., \cite{sutskever2013importance} and \cite{bengio2013advances}. 

\subsection{Nesterov Accelerated Gradient}
NAG \cite{nesterov1983method,beck2009fast} replaces the constant $\mu$ with $(t_k-1)/t_{k+1}$, where $t_{k+1}=(1+\sqrt{1+4t_k^2})/2$ with $t_0=1$, 
{
\begin{eqnarray}
\label{eq:NAG}
\begin{aligned}
  \vv^{k+1} &= \vw^k -  s_k\nabla f(\vw^k),\\
  \vw^{k+1} &= \vv^{k+1} + \frac{t_k-1}{t_{k+1}} (\vv^{k+1} - \vv^{k}).
\end{aligned}
\end{eqnarray}}
NAG achieves a convergence rate $O(1/k^2)$ with the step size $s_k=1/L$, which is the optimal rate for general convex smooth optimization problems.

\begin{remark}
Su et al. \cite{su2014differential} showed that $(k-1)/(k+2)$ is the asymptotic limit of $(t_k-1)/t_{k+1}$. In the following presentation of NAG with restart, for the ease of notation, we will replace the momentum coefficient $(t_k-1)/t_{k+1}$ with the form of $(k-1)/(k+2)$.
\end{remark}

\subsection{Adaptive Restart NAG (ARNAG)}
The sequences, $\{f(\vw^k)-f(\vw^*)\}$ where $\vw^*$ is the minimum of $f(\vw)$, generated by GD and GD with constant momentum (GD $+$ Momentum) converge monotonically to zero. However, that sequence generated by NAG oscillates, 
as illustrated in Fig.~\ref{fig:Quadratic-Compare-Optimization} (a) 
when $f(\vw)$ is a quadratic function.
\cite{o2015adaptive} proposes ARNAG \eqref{eq:ARNAG} to alleviate this oscillatory phenomenon
{\begin{eqnarray}
\label{eq:ARNAG}
\begin{aligned}
  \vv^{k+1} &= \vw^k -  s_k\nabla f(\vw^k),\\
  \vw^{k+1} &= \vv^{k+1} + \frac{m(k)-1}{m(k)+2} (\vv^{k+1} - \vv^{k}),
\end{aligned}
\end{eqnarray}}
where $m(1)=1$; $m(k+1)=m(k)+1$ if $f(\vw^{k+1}) \leq f(\vw^{k})$, and $m(k+1)=1$ otherwise.

\subsection{Scheduled Restart NAG (SRNAG)} \label{sec:SRNAG}
SR is another strategy to restart NAG. We first divide the total iterations $(0, T]$ (integers only) into a few intervals $\{I_i\}_{i=1}^m = (T_{i-1}, T_i]$, such that $(0, T] = \bigcup_{i=1}^m I_i$. In each $I_i$ we restart the momentum after every $F_i$, and the iteration is according to:
{\begin{eqnarray}
\label{eq:SRNAG}
\begin{aligned}
  \vv^{k+1} &= \vw^k -  s_k\nabla f(\vw^k),\\
  \vw^{k+1} &= \vv^{k+1} + \frac{(k\bmod F_i)}{(k\bmod F_i) + 3} (\vv^{k+1} - \vv^{k}).
\end{aligned}
\end{eqnarray}}
Both AR and SR accelerate NAG to linear convergence for convex problems with PL 
condition \cite{roulet2017sharpness}.

\subsection{Case Study -- Quadratic Function}
Consider the following quadratic optimization\footnote{We take this example from \cite{mrtz}.}
{\begin{equation}\label{eq:quadratic}
\min_\vx f(\vx) = \frac{1}{2}\vx^T\mathbf{L}\vx - \vx^T\vb,
\end{equation}}
where $\mathbf{L} \in \mathbb{R}^{d\times d}$ is the Laplacian of a cycle graph.
and $\vb$ is a $d$-dimensional vector whose first entry is $1$ and all the other entries are $0$. It is easy to see that $f(\vx)$ is convex with Lipschitz constant $4$. In particular, we set $d=1$K ($1$K$ := 10^3$). We run $T=50$K iterations with step size $1/4$. In SRNAG, we restart, i.e., we set the momentum to 0, after every $1$K iterations. 
As shown in Fig.~\ref{fig:Quadratic-Compare-Optimization} (a), GD $+$ Momentum converges faster than GD, while NAG speeds up GD $+$ Momentum dramatically and converges to the minimum in an oscillatory fashion. Both AR and SR accelerate NAG significantly.

\begin{figure}[!ht]
\centering
\includegraphics[width=1.0\linewidth]{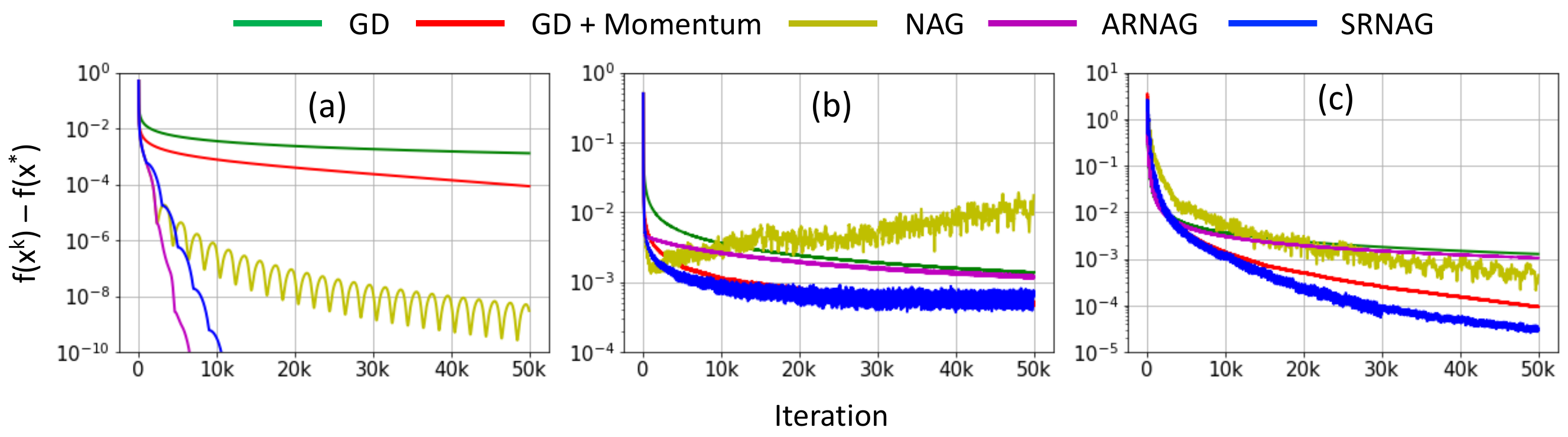}
\caption{Comparison between different schemes in optimizing the quadratic function, 
\eqref{eq:quadratic}, with (a) exact gradient, (b) gradient with constant variance Gaussian noise, and (c) gradient with decaying variance Gaussian noise. NAG, ARNAG, and SRNAG can speed up convergence remarkably when exact gradient is used. 
Also, SRNAG is more robust to noisy gradient than NAG and ARNAG.
}
\label{fig:Quadratic-Compare-Optimization}
\end{figure}

\section{Scheduled Restart SGD (SRSGD)}\label{section:SRSGD}
Computing 
gradient for 
ERM, \eqref{eq:finite-sum}, can be computational costly and memory intensive, especially when the training set is large. In many 
applications, such as training DNNs, 
SGD \eqref{eq:SGD} is used. 
In this section, we will first analyze whether NAG and restart techniques can still speed up SGD.
Then we formulate our new SRSGD as a solution to accelerate convergence of SGD using NAG momentum.

\subsection{Uncontrolled Bound of Nesterov Accelerated SGD (NASGD)}
Replacing {\small$\nabla f(\vw^k) :=1/N\sum_{i=1}^N\nabla f_i(\vw^k)$} in \eqref{eq:NAG} with the stochastic gradient {\small$1/m\sum_{j=1}^m\nabla f_{i_j}(\vw^k)$} for \eqref{eq:finite-sum} will accumulate error even for convex function. We formulate this fact in Theorem~\ref{Theorem:divergence-NASGD-Main}.

\begin{theorem}
\label{Theorem:divergence-NASGD-Main}
Let $f(\vw)$ be a convex and $L$-smooth function. The sequence $\{\vw^k\}_{k\geq 0}$ generated by \eqref{eq:NAG}, with mini-batch stochastic gradient using any constant step size $s_k \equiv s \leq 1/L$, satisfies
\begin{equation}
\label{eq:divergence:NASGD-result}
{\small \mathbb{E}\left(f(\vw^k)-f(\vw^*)\right) = O(k),}
\end{equation}
where $\vw^*$ is the minimum of $f$, and the expectation is taken over the random mini-batch samples.
\end{theorem}

In Appendix~\ref{appendix:uncontrolled:NASGD}, we provide the proof of Theorem~\ref{Theorem:divergence-NASGD-Main}. 
In \cite{devolder2014first}, Devolder et al. proved a similar error accumulation result for the $\delta$-inexact gradient. In Appendix~\ref{appendix:Delta:Inexact:Gradient}, we provide a brief review of NAG with $\delta$-inexact gradient. We consider three different inexact gradients, namely, Gaussian noise with constant and decaying variance corrupted gradients for the quadratic optimization \eqref{eq:quadratic}, and training logistic regression model for MNIST \cite{lecun-mnisthandwrittendigit-2010} classification. The detailed settings and discussion are provided in the Appendix~~\ref{appendix:Delta:Inexact:Gradient}. We denote SGD with NAG momentum as NASGD, and denote NASGD with AR and SR as ARSGD and SRSGD, respectively. The results shown in Fig.~\ref{fig:Quadratic-Compare-Optimization} (b) and (c) (iteration vs.\ optimal gap for quadratic optimization \eqref{eq:quadratic}
), and Fig.~\ref{fig:Logistic-Regression-Comparison} (iteration vs. loss for training logistic regression model) confirm Theorem~\ref{Theorem:divergence-NASGD-Main}. Moreover, for these cases SR can improve the performance of NAG with inexact gradients. 
When inexact gradient is used, GD performs almost the same as ARNAG asymptotically because ARNAG restarts too often and almost degenerates to GD.


\begin{figure}[!ht]
\centering
\begin{tabular}{c}
\includegraphics[width=0.65\columnwidth]{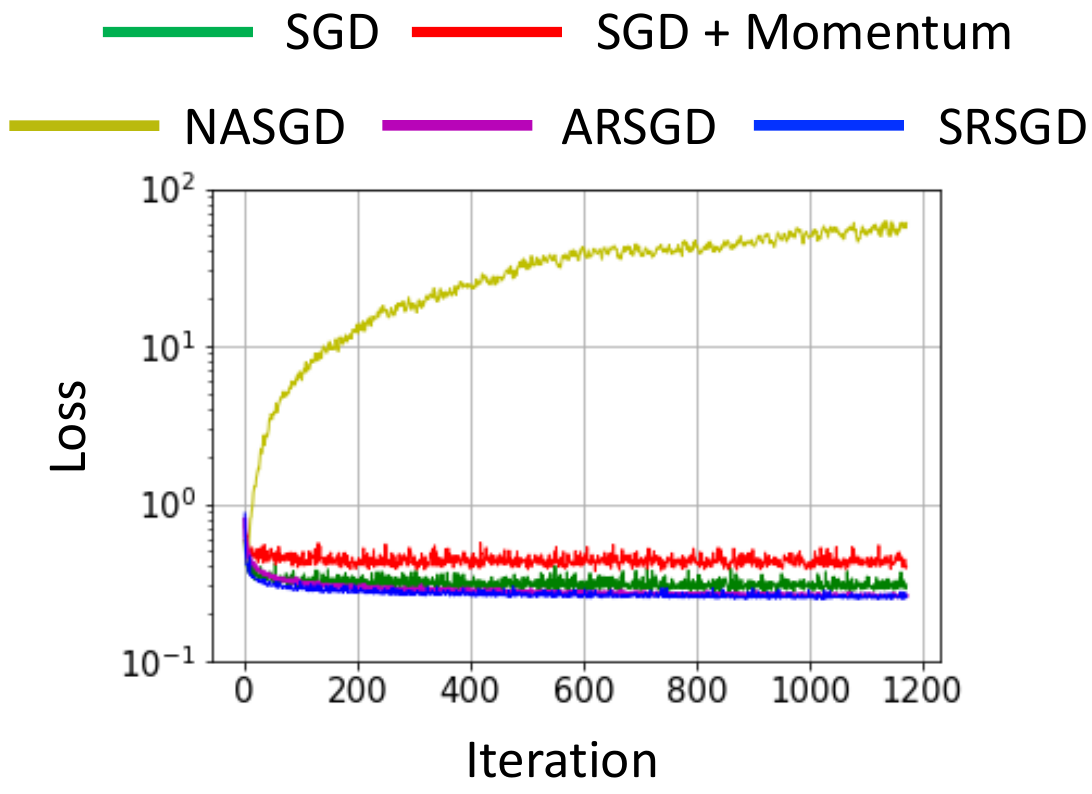}\\
\end{tabular}
\caption{Training loss comparison between different schemes in training logistic regression for MNIST classification. NASGD is not robust to noisy gradient, ARSGD almost degenerates to SGD, and SRSGD performs the best in this case.}
\label{fig:Logistic-Regression-Comparison}
\end{figure}

\subsection{SRSGD and Its Convergence}\label{subsection:SRSGD}
For ERM \eqref{eq:finite-sum}, SRSGD replaces $\nabla f(\vw)$ in \eqref{eq:SRNAG} with the stochastic gradient with batch size $m$, gives
{\small \begin{eqnarray}
\label{eq:SRSGD}
\begin{aligned}
  \vv^{k+1} &= \vw^k -  s_k\frac{1}{m}\sum_{j=1}^m\nabla f_{i_j}(\vw^k),\\
  \vw^{k+1} &= \vv^{k+1} + \frac{(k\bmod F_i)}{(k\bmod F_i) + 3} (\vv^{k+1} - \vv^{k}),
\end{aligned}
\end{eqnarray}}
where $F_i$ is the restart frequency used in the interval $I_i$. 
We implemented SRSGD in both PyTorch \cite{paszke2019pytorch} and Keras \cite{chollet2015keras}, by changing just a few lines on top of the existing SGD 
optimizer. We provide a snippet of SRSGD code in Appendix~\ref{appendix:PyTorch} and \ref{appendix:Keras}. We formulate the convergence of SRSGD for general 
nonconvex problems 
in Theorem~\ref{Theorem:convergence:SRSGD:main} and we provide its proof in Appendix~\ref{appendix:Convergence:SRSGD}.

\begin{theorem}
\label{Theorem:convergence:SRSGD:main}
Suppose $f(\vw)$ is $L$-smooth. Consider the sequence $\{\vw^k\}_{k\geq 0}$ generated by \eqref{eq:SRSGD} with mini-batch stochastic gradient and any restart frequency $F$ using any constant step size $s_k := s\leq 1/L$. Assume that the set $\mathcal{A} := \{k\in \mathbb{Z}^+| \mathbb{E}f(\vw^{k+1})\geq \mathbb{E}f(\vw^k) \}$ is finite, then we have 
{\small \begin{equation}
\label{eq:convergence:SRSGD}
\min_{1\leq k\leq K}\left\{\mathbb{E}\|\nabla f(\vw^k)\|_2^2\right\} = O(s+ \frac{1}{sK}).
\end{equation}}
Therefore for $\forall \epsilon>0$, to get $\epsilon$ error, we just need to set $s=O(\epsilon)$ and $K=O(1/\epsilon^2)$.
\end{theorem}

\section{Experimental Results}\label{section:experiments}
\renewcommand\arraystretch{1.3}
\setlength{\tabcolsep}{4pt}
\begin{table*}[t!]
\caption{Classification test error (\%) on CIFAR10 using the SGD, 
SGD $+$ NM, and SRSGD. We report the results of SRSGD with two different restarting schedules: linear (lin) and exponential (exp). The numbers of iterations after which we restart the momentum in the lin schedule are 30, 60, 90, 120 for the 1st, 2nd, 3rd, and 4th stage. 
Those numbers for the exp schedule are 40, 50, 63, 78. We also include the reported 
results from \cite{he2016identity} (in parentheses) in addition to our reproduced results.}\label{tab:cifar10-acc}
\vspace{2mm}
\centering
\fontsize{7.5pt}{0.75em}\selectfont
\begin{tabular}{c|c|c|c|c|c|c|c}
\hline
Network & \# Params&SGD (baseline) & SGD$+$NM & SRSGD & SRSGD  & Improve over & Improve over \\
 & & & & (lin) & (exp) & SGD (lin/exp) & SGD$+$NM (lin/exp) \\
\hline
Pre-ResNet-110 &$1.1$M & $5.25 \pm 0.14$ ($6.37$) & $5.24 \pm 0.16$ & \pmb{$4.93 \pm 0.13$} & $5.00 \pm 0.47$ & $\pmb{0.32}/0.25$ & $\pmb{0.31}/0.24$ \\
Pre-ResNet-290 &$3.0$M & $5.05 \pm 0.23$ & $5.04 \pm 0.12$ & \pmb{$4.37 \pm 0.15$} & $4.50 \pm 0.18$ & $\pmb{0.68}/0.55$ & $\pmb{0.67}/0.54$ \\
Pre-ResNet-470 &$4.9$M & $4.92 \pm 0.10$ & $4.97 \pm 0.15$ & \pmb{$4.18 \pm 0.09$} & $4.49 \pm 0.19$ & $\pmb{0.74}/0.43$ & $\pmb{0.79}/0.48$ \\
Pre-ResNet-650 &$6.7$M & $4.87 \pm 0.14$ & $4.80 \pm 0.14$ & \pmb{$4.00 \pm 0.07$} & $4.40 \pm 0.13$ & $\pmb{0.87}/0.47$ & $\pmb{0.80}/0.40$ \\
Pre-ResNet-1001&$10.3$M & $4.84 \pm 0.19$ ($4.92$) & $4.62 \pm 0.14$ & \pmb{$3.87 \pm 0.07$} & $4.13 \pm 0.10$ & $\pmb{0.97}/0.71$ & $\pmb{0.75}/0.49$ \\
\hline
\end{tabular}
\end{table*}

We evaluate SRSGD on a variety of DL benchmarks for image classification, including CIFAR10, CIFAR100, and ImageNet. In all experiments, we show the advantage of SRSGD over the widely used and well-calibrated SGD baselines with a constant momentum of $0.9$ and decreasing learning rate at certain epochs, and we denote this optimizer as SGD. We also compare SRSGD with the well-calibrated SGD but switch momentum to the Nesterov momentum 
of $0.9$, 
and we denoted this optimizer as SGD $+$ NM. 
We fine tune the SGD/SGD $+$ NM baselines to obtain the best performance, and we then adopt the same set of parameters for training with SRSGD. In the SRSGD experiments, we tune the restart frequencies on small DNNs and apply the tuned restart frequencies to large DNNs. We provide the detailed description of datasets and experimental settings in Appendix~\ref{sec:all-datasets-implementation}.

\subsection{CIFAR10 and CIFAR100}
\label{sec:cifar-exp}
We summarize our results for CIFAR in Table~\ref{tab:cifar10-acc} and~\ref{tab:cifar100-acc}. We also explore two different restarting frequency schedules for SRSGD: \emph{linear} and \emph{exponential} schedule. These schedules are governed by two parameters: the initial restarting frequency $F_{1}$ and the growth rate $r$. In both scheduling schemes, during training, the restarting frequency at the 1st learning rate stage is set to $F_{1}$. Then the restarting frequency at the $(k+1)$-th learning rate stage is determined by: 
{\begin{align}
    F_{k+1} = 
    \begin{cases}
        F_{1} \times r^{k}, &\text{exponential schedule} \\
        F_{1} \times (1 + (r-1) \times k), &\text{linear schedule}.
    \end{cases}\nonumber
\end{align}}
We have conducted a hyper-parameter search for $F_{1}$ and $r$ for both scheduling schemes. For CIFAR10, $(F_{1} = 40, r = 1.25)$ and $(F_{1} = 30, r = 2)$ are good initial restarting frequencies and growth rates for the exponential and linear schedules, respectively. For CIFAR100, those values are $(F_{1} = 45, r = 1.5)$ for the exponential schedule and $(F_{1} = 50, r = 2)$ for the linear schedule. 

\renewcommand\arraystretch{1.3}
\setlength{\tabcolsep}{4pt}
\begin{table*}[t!]
\caption{Classification test error (\%) on CIFAR100 using the SGD, SGD $+$ NM, and SRSGD. We report the results of SRSGD with two different restarting schedules: linear (lin) and exponential (exp). The numbers of iterations after which we restart the momentum in the lin schedule are 50, 100, 150, 200 for the 1st, 2nd, 3rd, and 4th stage. Those numbers for the exp schedule are 45, 68, 101, 152. We also include the reported 
results from \cite{he2016identity} (in parentheses) in addition to our reproduced results.}\label{tab:cifar100-acc}
\vspace{2mm}
\centering
\fontsize{7.5pt}{0.75em}\selectfont
\begin{tabular}{c|c|c|c|c|c|c|c}
\hline
Network & \# Params&SGD (baseline) & SGD$+$NM & SRSGD & SRSGD  & Improve over & Improve over \\
 & & & & (lin) & (exp) & SGD (lin/exp) & SGD$+$NM (lin/exp) \\
\hline
Pre-ResNet-110 &$1.2$M & $23.75 \pm 0.20$ & $23.65 \pm 0.36$ &\pmb{$23.49 \pm 0.23$} & $23.50 \pm 0.39$ & $\pmb{0.26}/0.25$ & $\pmb{0.16}/0.15$ \\
Pre-ResNet-290 &$3.0$M & $21.78 \pm 0.21$ & $21.68 \pm 0.21$ & \pmb{$21.49 \pm 0.27$} & $21.58 \pm 0.20$ & $\pmb{0.29}/0.20$ & $\pmb{0.19}/0.10$ \\
Pre-ResNet-470 &$4.9$M & $21.43 \pm 0.30$ & $21.21 \pm 0.30$ & $20.71 \pm 0.32$ & \pmb{$20.64 \pm 0.18$} & $0.72/\pmb{0.79}$ & $0.50/\pmb{0.57}$ \\
Pre-ResNet-650 &$6.7$M & $21.27 \pm 0.14$ & $21.04 \pm 0.38$ &\pmb{$20.36 \pm 0.25$}& $20.41 \pm 0.21$ & $\pmb{0.91}/0.86$ & $\pmb{0.68}/0.63$ \\
Pre-ResNet-1001&$10.4$M & $20.87 \pm 0.20$ ($22.71$) & $20.13 \pm 0.16$ & $19.75 \pm 0.11$ & \pmb{$19.53 \pm 0.19$} & $1.12/\pmb{1.34}$ & $0.38/\pmb{0.60}$\\
\hline
\end{tabular}
\end{table*}

{\bf Improvement in Accuracy Increases with Depth:} We observe that the linear schedule of restart 
yields better test error on CIFAR than the exponential schedule for most of the models except for Pre-ResNet-470 and Pre-ResNet-1001 on CIFAR100 (see Table~\ref{tab:cifar10-acc} and~\ref{tab:cifar100-acc}). SRSGD with either linear or exponential restart schedule 
outperforms the SGD.
Furthermore, the advantage of SRSGD over SGD is greater for deeper networks. This observation holds strictly when using the linear schedule (see Fig.~\ref{fig:error-vs-depth}) and is overall true when using the exponential schedule with only a few exceptions.

\begin{figure}[h!]
\centering
\includegraphics[width=1.0\linewidth]{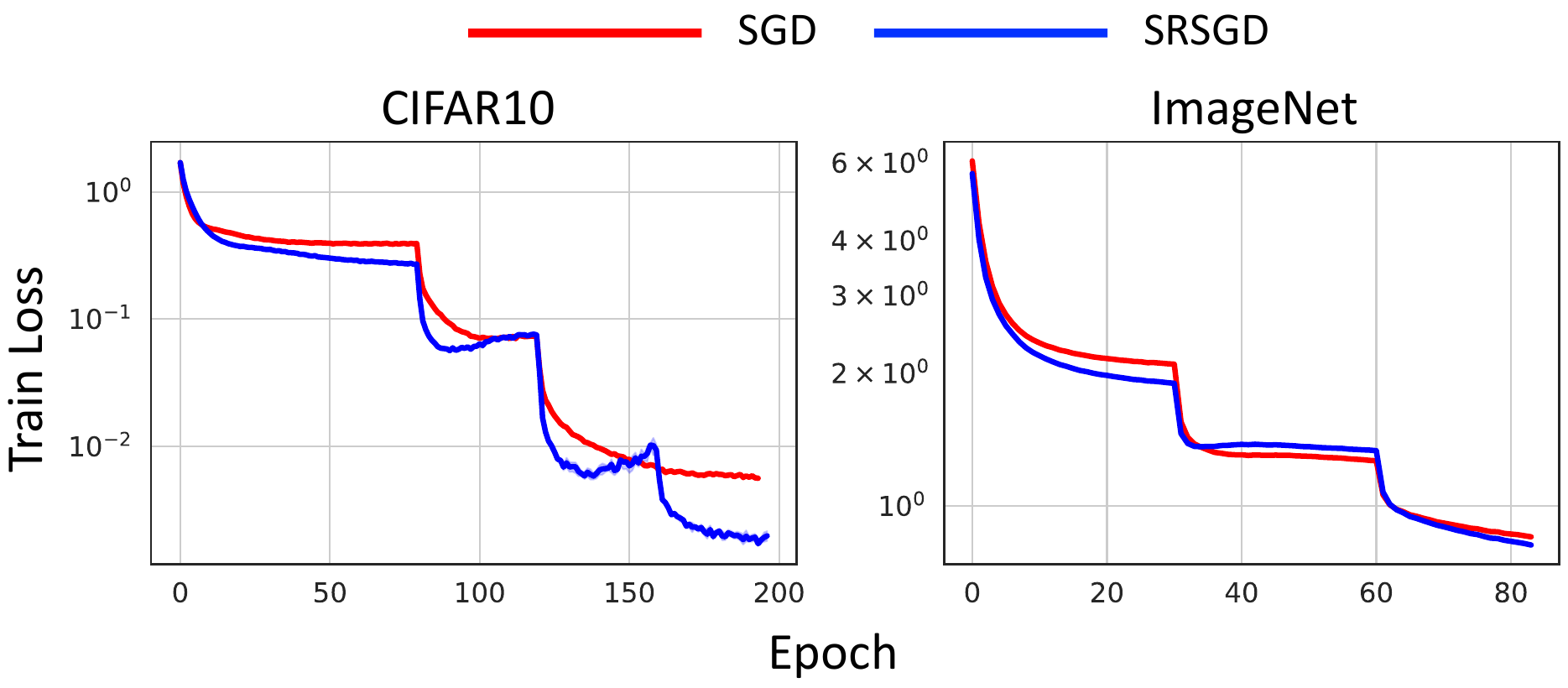}
\caption{Training loss vs. training epoch of ResNet models trained with SRSGD (blue) and the SGD baseline with momentum (red).}
\label{fig:train-loss-vs-epoch}
\end{figure}

{\bf Faster Convergence Reduces the Training Time by Half:} SRSGD also converges faster than SGD. 
This is expected since 
SRSGD can avoid the error accumulation with inexact oracle and converges faster than SGD + Momentum in our MNIST case study in Section~\ref{section:SRSGD}. For CIFAR, Fig.~\ref{fig:train-loss-vs-epoch} (left) shows that SRSGD yields smaller training loss than SGD during the training. Interestingly, SRSGD converges 
quickly to good loss values at the 2nd and 3rd stages. 
This suggests that \emph{the model can be trained with SRSGD in many fewer epochs compared to SGD} while achieving similar error rate. 

Our numerical results in Table~\ref{tab:cifar-short-train} confirm the hypothesis above. We train Pre-ResNet models with SRSGD \emph{in only 100 epochs}, decreasing the learning rate by a factor of 10 at the 80th, 90th, and 95th epoch while using the same linear schedule for restarting frequency as before with $(F_{1} = 30, r = 2)$ for CIFAR10 and $(F_{1} = 50, r = 2)$ for CIFAR100. We compare the test error of the trained models with those trained by the SGD baseline in 200 epochs. We observe that SRSGD trainings consistently yield lower test errors than SGD except for the case of Pre-ResNet-110 even though \emph{the number of training epochs of our method is only half of the number of training epochs required by SGD}. For Pre-ResNet-110, SRSGD training in 110 epochs with learning rate decreased at the 80th, 90th, and 100th epoch achieves the same error rate as the 200-epoch SGD training on CIFAR10. On CIFAR100, SRSGD training for Pre-ResNet-110 needs 140 epochs with learning rate decreased at the 80th, 100th and 120th epoch to achieve an 0.02\% improvement in error rate over the 200-epoch SGD. 

\renewcommand\arraystretch{1.3}
\setlength{\tabcolsep}{4pt}
\begin{table}[t!]
\caption{Comparison of classification errors on CIFAR10/100 (\%) between SRSGD training \emph{with only 100 epochs} and SGD baseline training with 200 epochs. Using only half the number of training epochs, SRSGD achieves comparable results to SGD.} \label{tab:cifar-short-train}
\vspace{2.0mm}
\centering
\fontsize{7.5pt}{0.75em}\selectfont
\begin{tabular}{c|c|c|c|c}
\hline
&\multicolumn{2}{c}{CIFAR10} & \multicolumn{2}{c}{CIFAR100}\\
\hline
Network & SRSGD & Improvement & SRSGD & Improvement \\
\hline
Pre-ResNet-110 & $5.43 \pm 0.18$ & $-0.18$ & $23.85 \pm 0.19$ & $-0.10$ \\
Pre-ResNet-290 & $4.83 \pm 0.11$ & $0.22$ & $21.77 \pm 0.43$ & $0.01$ \\
Pre-ResNet-470 & $4.64 \pm 0.17$ & $0.28$ & $21.42 \pm 0.19$ & $0.01$ \\
Pre-ResNet-650 & $4.43 \pm 0.14$ & $0.44 $ & $21.04 \pm 0.20$ & $0.23$  \\
Pre-ResNet-1001 & $4.17 \pm 0.20$ & $0.67$ & $20.27 \pm 0.11$ & $ 0.60$ \\
\hline
Pre-ResNet-110 & $5.25 \pm 0.10$ (110 epochs) & $0.00$ & $23.73 \pm 0.23$ (140 epochs) & $0.02$ \\
\hline
\end{tabular}
\end{table}

\subsection{ImageNet}
Next we discuss our experimental results on the 1000-way ImageNet classification task \cite{russakovsky2015imagenet}. 
We conduct our 
experiments on ResNet-50, 101, 152, and 200 with 5 different seeds. We use the official Pytorch implementation\footnote{Implementation available at \href{https://github.com/pytorch/examples/tree/master/imagenet}{https://github.com/pytorch/examples/tree/master/imagenet}} for all of our ResNet models \cite{paszke2019pytorch}. Following common practice, we train each model for 90 epochs and decrease the learning rate by a factor of 10 at the 30th and 60th epoch. We use an initial learning rate of 0.1, momentum value of 0.9, and weight decay value of 0.0001. Additional details and comparison between SRSGD and SGD $+$ NM are given in Appendix~\ref{appendix:ImageNet:Others}.

We report single crop validation
errors of ResNet models trained with SGD and SRSGD on ImageNet in Table~\ref{tab:imagenet-acc}. In contrast to our CIFAR experiments, we observe that for ResNets trained on ImageNet with SRSGD, linearly decreasing the restarting frequency to 1 at the last learning rate (i.e., after the 60th epoch) helps improve the generalization of the models. Thus, in our experiments, we set the restarting frequency to a linear schedule until epoch 60. From epoch 60 to 90, the restarting frequency is linearly decreased to 1. We use $(F_{1} = 40, r = 2)$.

\renewcommand\arraystretch{1.3}
\setlength{\tabcolsep}{4pt}
\begin{table*}[!ht]
\caption{Single crop validation
errors (\%) on ImageNet of ResNets trained with SGD baseline and SRSGD. We report the results of SRSGD with the increasing restarting frequency in the first two learning rates. In the last learning rate, the restarting frequency is linearly decreased from 70 to 1. For baseline results, we also include the reported single-crop validation errors 
\cite{he@github} (in parentheses).}\label{tab:imagenet-acc}
\vspace{2mm}
\centering
\fontsize{7.5pt}{0.75em}\selectfont
\begin{tabular}{c|c|c|c|c|c|c|c}
\hline
Network & \# Params & \multicolumn{2}{c|}{SGD} & \multicolumn{2}{c|}{SRSGD} & \multicolumn{2}{c}{Improvement} \\
\hline
 &  & top-1 & top-5  & top-1 & top-5  & top-1 & top-5  \\
\hline
ResNet-50 & $25.56$M & $24.11 \pm 0.10$ ($24.70$) & $7.22 \pm 0.14$ ($7.80$) & \pmb{$23.85 \pm 0.09$} & \pmb{$7.10 \pm 0.09$} & $0.26$ & $0.12$ \\
ResNet-101 & $44.55$M & $22.42 \pm 0.03$ ($23.60$) & $6.22 \pm 0.01$ ($7.10$) & \pmb{$22.06 \pm 0.10$} & \pmb{$6.09 \pm 0.07$} & $0.36$ & $0.13$ \\
ResNet-152 & $60.19$M & $22.03 \pm 0.12$ ($23.00$) & $6.04 \pm 0.07$ ($6.70$) & \pmb{$21.46 \pm 0.07$} & \pmb{$5.69 \pm 0.03$} & $0.57$ & $0.35$ \\
ResNet-200 & $64.67$M & $22.13 \pm 0.12$ & $6.00 \pm 0.07$ & \pmb{$20.93 \pm 0.13$} & \pmb{$5.57 \pm 0.05$} & $1.20$ & $0.43$ \\
\hline
\end{tabular}
\end{table*}

{\bf Advantage of SRSGD continues to grow with depth:} Similar to the CIFAR experiments, we observe that SRSGD outperforms the SGD baseline for all ResNet models that we study. As shown in Fig.~\ref{fig:error-vs-depth}, \emph{the advantage of SRSGD over SGD grows with network depth}, just as in our CIFAR experiments with Pre-ResNet architectures. 

{\bf Avoiding Overfitting in ResNet-200:} ResNet-200 is an interesting model that demonstrates that \emph{SRSGD is better than the SGD baseline at avoiding overfitting.\footnote{By overfitting, we mean that the model achieves low training error but high test error.}} The ResNet-200 trained with SGD has a top-1 error of 22.18\%, higher than the ResNet-152 trained with SGD, which achieves a top-1 error of 21.9\% (see Table~\ref{tab:imagenet-acc}). As pointed out in \cite{he2016identity}, it is because ResNet-200 suffers from overfitting. The ResNet-200 trained with our SRSGD has a top-1 error of 21.08\%, which is 1.1\% lower than the ResNet-200 trained with the SGD baseline and also lower than the ResNet-152 trained with  both SRSGD and SGD, an improvement by 0.21\% and 0.82\%, respectively.  

\renewcommand\arraystretch{1.3}
\setlength{\tabcolsep}{4pt}
\begin{table}[h!]
\caption{Comparison of single crop validation
errors on ImageNet (\%) between SRSGD training \emph{with fewer epochs} and SGD training with full 90 epochs.}\label{tab:imagenet-short-train}
\vspace{2mm}
\centering
\fontsize{7.5pt}{0.75em}\selectfont
\begin{tabular}{c|c|c|c|c|c|c|c}
\hline
Network & SRSGD & Reduction & Improvement & Network & SRSGD & Reduction & Improvement \\
\hline
ResNet-50 & $24.30 \pm 0.21$ & $10$ & $-0.19$ & ResNet-152 & $21.79 \pm 0.07$ & $15$ & $0.24$ \\
ResNet-101 & $22.32 \pm 0.06$ & $10$ & $0.1$ & ResNet-200 & $21.92 \pm 0.17$ & $30$ & $0.21$ \\
\hline
\end{tabular}
\end{table}

{\bf Training ImageNet in Fewer Number of Epochs: } As in the CIFAR experiments, we note that when training on ImageNet, SRSGD converges faster than SGD at the first and last learning rate while quickly reaching a good loss value at the second learning rate (see Fig.~\ref{fig:train-loss-vs-epoch}). This observation suggests that ResNets can be trained with SRSGD in fewer epochs while still achieving comparable error rates to the same models trained by the SGD baseline using all 90 epochs. We summarize the results in Table~\ref{tab:imagenet-short-train}. On ImageNet, we note that SRSGD helps reduce the number of training epochs for very deep networks (ResNet-101, 152, 200). For smaller networks like ResNet-50, training with fewer epochs slightly decreases the accuracy.



\section{Empirical Analysis}\label{section:Empirical-Analysis}

\paragraph{Error Rate vs. Reduction in Epochs.}
\begin{figure}[!ht]
\centering
\includegraphics[width=1.0\linewidth]{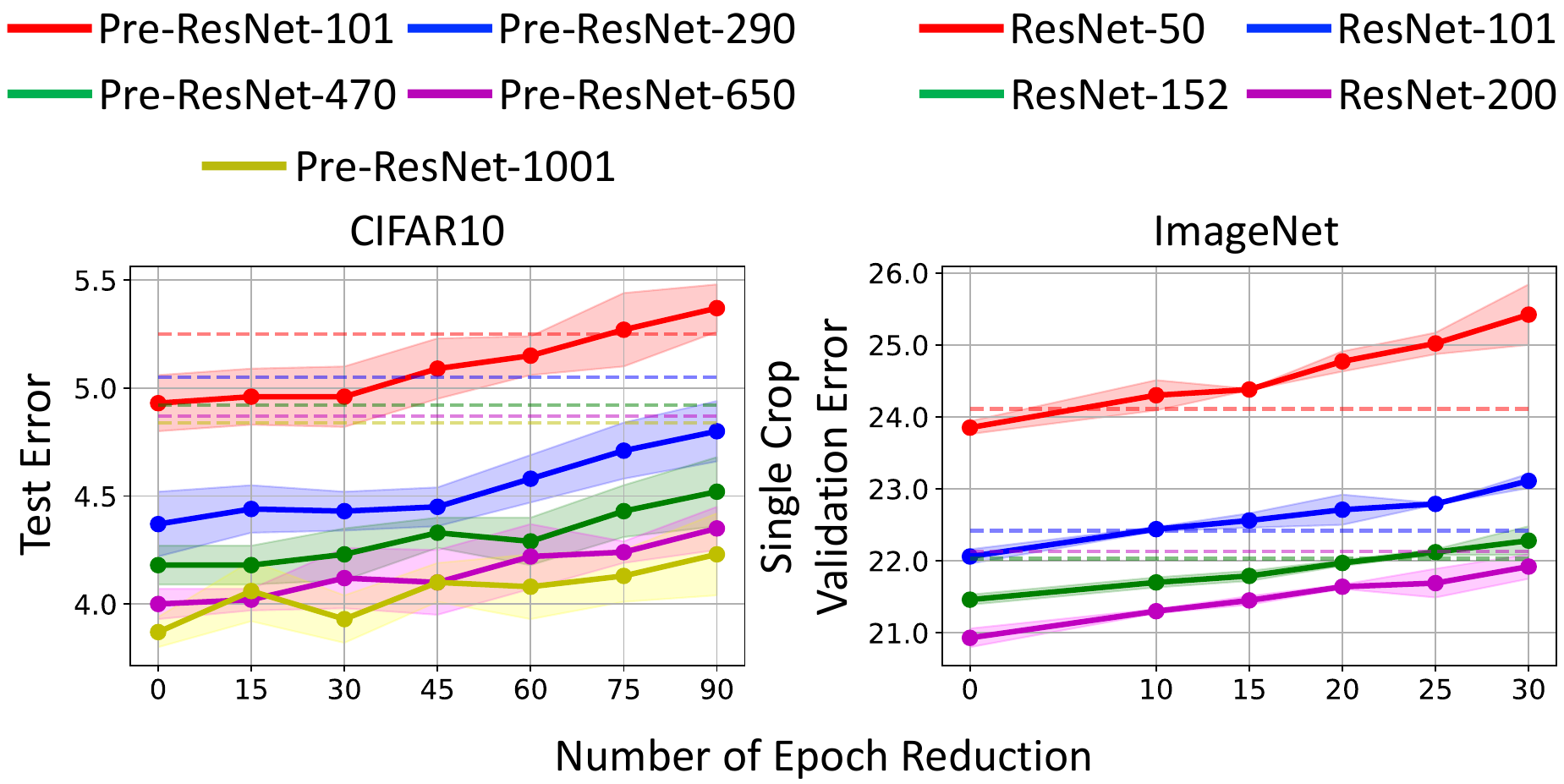}
\caption{Test error vs.\ number of epoch reduction in CIFAR10 and ImageNet training. The dashed lines are test errors of the SGD baseline. For CIFAR, SRSGD training with fewer epochs can achieve comparable results to SRSGD training with full 200 epochs. For ImageNet, training with less epochs slightly decreases the performance of SRSGD but still achieves comparable results to the SGD baseline training.}
\label{fig:error-vs-epoch-reduction}
\end{figure}

We find that SRSGD training using fewer epochs yield comparable error rate to both the SGD baseline and the SRSGD full training with 200 epochs on CIFAR. We conduct an ablation study to understand the impact of reducing the number of epochs on the final error rate when training with SRSGD on CIFAR10 and ImageNet. In the CIFAR10 experiments, we reduce the number of epochs from 15 to 90 while in the ImageNet experiments, we reduce the number of epochs from 10 to 30.
We summarize our results in Fig.~\ref{fig:error-vs-epoch-reduction} and provide detailed results in Appendix~\ref{sec:error-rate-vs-epoch-reduction-appendix}. For CIFAR10, we can train with 30 epochs less while still maintaining a comparable error rate to the full SRSGD training, and with a better error rate than the SGD baseline. For ImageNet, SRSGD training with fewer epochs decreases the accuracy but still obtains comparable results to the 90-epoch SGD baseline as shown in Table~\ref{tab:imagenet-short-train}.

\paragraph{Impact of Restarting Frequency}
We examine the impact of restarting frequency on the network training. We choose a case study of training Pre-ResNet-290 on CIFAR10 using SRSGD with a linear schedule scheme for the restarting frequency. We fix the growth rate $r=2$ and vary the initial restarting frequency $F_{1}$ from 1 to 80 in increments of 10. As shown in Fig.~\ref{fig:ablation_restarting_frequency},  SRSGD with large $F_{1}$, e.g. $F_{1}=80$, approximates NASGD (yellow). As discussed in Section~\ref{section:SRSGD}, it suffers from error accumulation due to stochastic gradients and converges slowly. SRSGD with small $F_{1}$, e.g. $F_{1}=1$, approximates SGD without momentum (green). It converges faster initially but reaches a worse local minimum (i.e. greater loss). Typical SRSGD (blue) converges faster than NASGD and to a better local minimum than both NASGD and SGD without momentum. It also achieves the best test error. We provide more results in Appendix~\ref{sec:impact-restarting-frequency-appendix} and~\ref{appendix:FullTrainingLessEpochs}.

\begin{figure}[!ht]
\centering
\begin{tabular}{c}
\includegraphics[width=1.0\columnwidth]{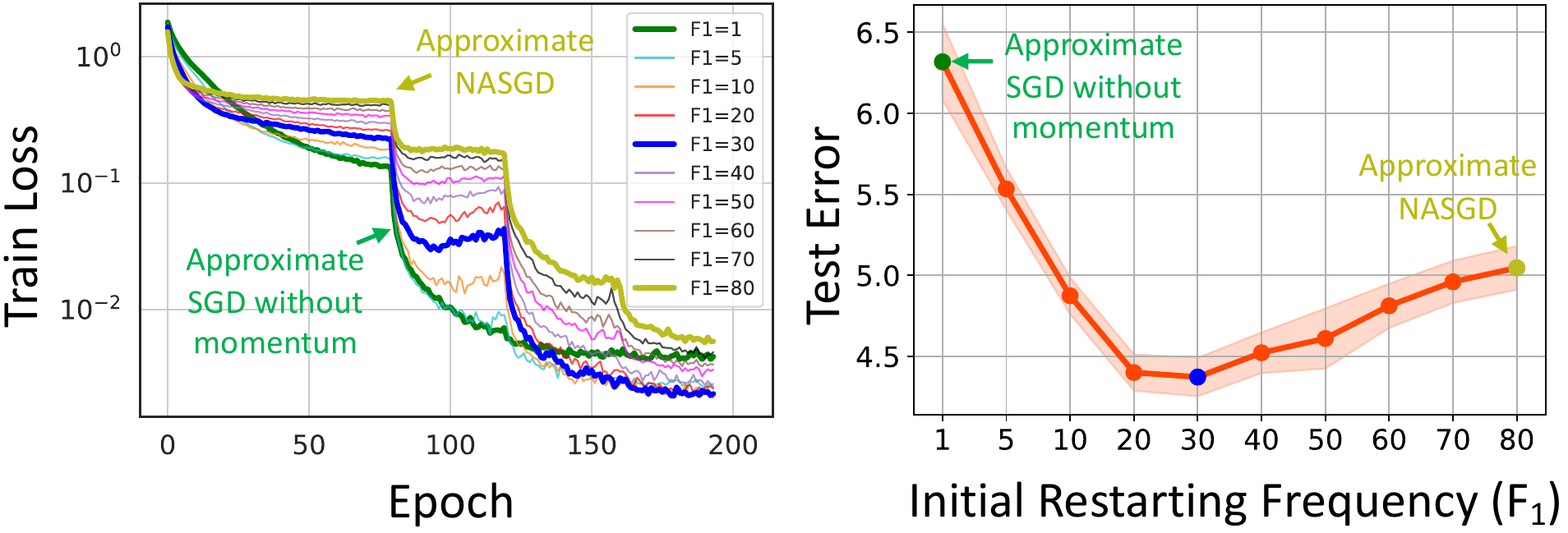}\\
\end{tabular}
\caption{Training loss and test error of Pre-ResNet-290 trained on CIFAR10 with different initial restarting frequencies $F_1$ (linear schedule). SRSGD with small $F_1$ approximates SGD without momentum, while SRSGD with large $F_1$ approximates NASGD.}
\label{fig:ablation_restarting_frequency}
\end{figure}

\section{Additional Related Work}\label{section:related_work}

Momentum has long been used to accelerate SGD. \cite{sutskever2013importance} showed that SGD with scheduled momentum and a 
good initialization can handle 
the curvature issues in training DNNs and enable the trained models to generalize well. \cite{kingma2014adam,dozat2016incorporating} integrated momentum with adaptive step size to accelerate SGD. These works all leverage constant momentum, while our work utilizes NAG momentum with restart. AR and SR have been used to accelerate NAG with exact gradient \cite{nemirovskii1985optimal,nesterov2013gradient,iouditski2014primal,lin2014adaptive,renegar2014efficient,freund2018new,roulet2015computational,o2015adaptive,giselsson2014monotonicity,su2014differential}. These studies of restart NAG momentum are for convex optimization with exact gradient. Our work focuses on SGD for nonconvex optimization.  Many efforts have also been devoted to accelerating first-order algorithms with noise-corrupted gradients \cite{cohen2018acceleration,aybat2018robust}.

\section{Conclusions}\label{section:conclusion}
We propose the Scheduled Restart SGD (SRSGD), with two major changes from the widely used SGD with constant momentum (without ambiguity we call it SGD). First, we replace the momentum in SGD with the increasing momentum in Nesterov accelerated gradient (NAG). Second, we restart the momentum according to a schedule to prevent error accumulation when the stochastic gradient is used. For image classification, SRSGD can significantly improve the accuracy of the trained DNNs. Also, compared to the SGD baseline, SRSGD requires fewer training epochs to reach to the same trained model's accuracy. There are numerous avenues for future work: 1) deriving the optimal restart scheduling and the corresponding convergence rate of SRSGD, 2) integrating the scheduled restart NAG momentum with adaptive learning rate algorithms, e.g. Adam, and 3) integrating SRSGD with optimizers that remove noise on the fly, e.g., Laplacian smoothing SGD \cite{osher2018laplacian}.

\section*{Acknowledgments}
This material is based on research sponsored by the National Science Foundation under grant number DMS-1924935 and DMS-1554564 (STROBE). 

\bibliography{example_paper.bib}
\bibliographystyle{plain}

\appendix
\newpage
\section{Uncontrolled Bound of NASGD}\label{appendix:uncontrolled:NASGD}
Consider the following optimization problem
\begin{equation}
\label{eq:divergence-NASGD:1}
\min_\vw = f(\vw) ,
\end{equation}
where $f(\vw)$ is $L$-smooth and convex. 

Start from $\vw^k$, GD update, with step size $\frac{1}{r}$, can be obtained based on the minimization of the functional
\begin{equation}
\label{eq:divergence-NASGD:2}
\min_\vv Q_r(\vv, \vw^k) := \langle \vv-\vw^k, \nabla f(\vw^k) \rangle  + \frac{r}{2}\|\vv - \vw^k\|_2^2.
\end{equation}

With direct computation, we can get that
$$Q_r(\vv^{k+1}, \vw^k) - \min Q_r(\vv, \vw^k)=\frac{\|{\bf g}^k-\nabla f({\bf w}^k)\|_2}{2r},$$
where ${\bf g}^k:=\frac{1}{m}\sum_{j=1}^m\nabla f_{i_j}(\vw^k)$. We assume the variance is bounded, which gives
The stochastic gradient rule, $\mathcal{R}_s$, satisfies $\mathbb{E}[Q_r(\vv^{k+1}, \vw^k) - \min Q_r(\vv, \vw^k) | \chi^k ]\leq \delta$,
with $\delta$ being a constant and $\chi^k$ being the sigma algebra generated by $\vw^1, \vw^2, \cdots, \vw^k$, i.e.,
$$
\chi^k := \sigma(\vw^1, \vw^2, \cdots, \vw^k).
$$

NASGD can be reformulated as

\begin{eqnarray}
\label{eq:divergence-NASGD:3}
\begin{aligned}
\vv^{k+1} &\approx& \min_\vv Q_r(\vv, \vw^k) \ \mbox{with rule}\ \mathcal{R}_s,\\
\vw^{k+1} &=& \vv^{k+1} + \frac{t_k-1}{t_{k+1}}(\vv^{k+1} - \vv^{k}),
\end{aligned}
\end{eqnarray}
where $t_0=1$ and $t_{k+1} = (1+\sqrt{1+4t_k^2})/2$.

\subsection{Preliminaries}
To proceed, we introduce several definitions and some useful properties in variational and convex analysis. More detailed background can be found at \cite{mordukhovich2006variational,nesterov1998introductory,rockafellar2009variational,rockafellar1970convex}. 

Let $f$ be a convex function, we say that $f$ is {\it $L$-smooth (gradient Lipschitz)} 
if $f$ is differentiable and
$$
\|\nabla f(\vv)-\nabla f(\vw)\|_2 \leq L\|\vv-\vw\|_2,
$$
and we say $f$ is {\it $\nu$-strongly convex} 
if for any $\vw, \vv \in {\rm dom}(J)$
$$
f(\vw) \geq f(\vv) + \langle \nabla f(\vv), \vw-\vv\rangle + \frac{\nu}{2}\|\vw-\vv\|_2^2.
$$
Below of this subsection, we list several basic but useful lemmas, the proof can be found in \cite{nesterov1998introductory}.

\begin{lemma}
\label{lemma:divergence-NASGD:2}
If $f$ is {\it $\nu$-strongly convex}, 
then for any $\vv\in {\rm dom}(J)$ we have
\begin{equation}
\label{eq:divergence-NASGD:5}
f(\vv) - f(\vv^*) \geq \frac{\nu}{2}\|\vv-\vv^*\|_2^2,
\end{equation}
where $\vv^*$ is the minimizer of $f$.
\end{lemma}

\begin{lemma}
\label{lemma:divergence-NASGD:3}
If $f$ is $L$-smooth, 
for any $\vw, \vv\in {\rm dom}(f)$,
$$
f(\vw) \leq f(\vv) + \langle \nabla f(\vv), \vw-\vv \rangle + 
\frac{L}{2}\|\vw-\vv\|_2^2.
$$
\end{lemma}

\subsection{Uncontrolled Bound of NASGD}
In this part, we denote
\begin{equation}
\label{eq:divergence-NASGD:7}
\Tilde{\vv}^{k+1} := \min_{\vv} Q_r(\vv, \vw^k).
\end{equation}

\begin{lemma}
\label{lemma:divergence-NASGD:4}
If the constant $r>0$, then
\begin{equation}
\label{eq:divergence-NASGD:8}
\mathbb{E}\left(\|\vv^{k+1} - \Tilde{\vv}^{k+1}\|_2^2 | \chi^k\right) \leq \frac{2\delta}{r}.
\end{equation}
\end{lemma}

\begin{proof}
Note that $Q_r(\vv, \vw^k)$ is strongly convex with constant $r$, and $\Tilde{\vv}^{k+1}$ in \eqref{eq:divergence-NASGD:7} is the minimizer of $Q_r(\vv, \vw^k)$. With Lemma~\ref{lemma:divergence-NASGD:2} we have
\begin{equation}
\label{eq:divergence-NASGD:9}
Q_r(\vv^{k+1}, \vw^k) - Q_r(\Tilde{\vv}^{k+1}, \vw^k) \geq \frac{r}{2} \|\vv^{k+1}-\Tilde{\vv}^{k+1}\|_2^2.
\end{equation}

Notice that
$$
\mathbb{E}\left[ Q_r(\vv^{k+1}, \vw^k) - Q_r(\Tilde{\vv}^{k+1}, \vw^k) \right] = \mathbb{E}\left[ Q_r(\vv^{k+1}, \vw^k) - \min_\vv Q_r(\vv, \vw^k) \right] \leq \delta.
$$

The inequality \eqref{eq:divergence-NASGD:8} can be established by combining the above two inequalities.
\end{proof}

\begin{lemma}
\label{lemma:divergence-NASGD:5}
If the constant satisfy $r>L$, then we have
\begin{eqnarray}
\label{eq:divergence:NASGD-10}
&&\mathbb{E}\left( f(\Tilde{\vv}^{k+1}) + \frac{r}{2}\|\Tilde{\vv}^{k+1}-\vw^k\|_2^2 - (f(\vv^{k+1}) + \frac{r}{2}\|\vv^{k+1}-\vw^k\|_2^2 )  \right) \\ \nonumber
&\geq& -\tau \delta - \frac{r-L}{2}\mathbb{E}[\|\vw^k-\Tilde{\vv}^{k+1}\|_2^2],
\end{eqnarray}
where 
$\tau = \frac{L^2}{r(r-L)}$.
\end{lemma}

\begin{proof}
The convexity of $f$ gives us
\begin{equation}
\label{eq:divergence:NASGD-11}
0\leq \langle \nabla f(\vv^{k+1}), \vv^{k+1} - \Tilde{\vv}^{k+1} \rangle + f(\Tilde{\vv}^{k+1}) - f(\vv^{k+1}),
\end{equation}

From the definition of the stochastic gradient rule $\mathcal{R}_s$, we have
\begin{eqnarray}
\label{eq:divergence:NASGD-12}
-\delta &\leq& \mathbb{E}\left(Q_r(\Tilde{\vv}^{k+1}, \vw^k ) - Q_r(\vv^{k+1}, \vw^k) \right)\\ \nonumber
&=& \mathbb{E}\left[\langle \Tilde{\vv}^{k+1}-\vw^k, \nabla f(\vw^k) \rangle + \frac{r}{2}\|\Tilde{\vv}^{k+1}-\vw^k \|_2^2 \right] - \\ \nonumber
&&\mathbb{E}\left[ \langle \vv^{k+1}-\vw^k, \nabla f(\vw^k) \rangle + \frac{r}{2}\|\vv^{k+1} - \vw^k \|_2^2 \right].
\end{eqnarray}
With \eqref{eq:divergence:NASGD-11} and \eqref{eq:divergence:NASGD-12}, we have
\begin{eqnarray}
\label{eq:divergence:NASGD-13}
-\delta &\leq& \left(f(\vv^{k+1})+\frac{r}{2}\|\vv^{k+1}-\vw^k\|_2^2 \right) - \left(f(\Tilde{\vv}^{k+1})+\frac{r}{2}\|\Tilde{\vv}^{k+1}-\vw^k\|_2^2 \right)  + \\ \nonumber
&&\mathbb{E}\langle \nabla f(\vw^k)-\nabla f(\Tilde{\vv}^{k+1}), \vv^{k+1}-\Tilde{\vv}^{k+1} \rangle.
\end{eqnarray}

With the Schwarz inequality $\langle \va, \vb\rangle \leq \frac{\|\va\|_2^2}{2\mu} + \frac{\mu}{2}\|\vb\|_2^2$ with $\mu = \frac{L^2}{r-L}$, $a=\nabla f(\vv^{k+1}) - \nabla f(\Tilde{\vv}^{k+1})$ and $b=\vw^k-\Tilde{\vv}^{k+1}$,
\begin{eqnarray}
\label{eq:divergnce:NASGD-14}
&&\langle \nabla f(\vw^k) - \nabla f(\Tilde{\vv}^{k+1}), \vv^{k+1}-\Tilde{\vv}^{k+1}\rangle \\ \nonumber
&\leq& \frac{(r-L)}{2L^2}\|\nabla f(\vw^k)-\nabla f(\Tilde{\vv}^{k+1}) \|_2^2 + \frac{L^2}{2(r-L)}\|\vv^{k+1} - \Tilde{\vv}^{k+1}\|_2^2 \\ \nonumber
&\leq& \frac{(r-L)}{2}\|\vw^k -\Tilde{\vv}^{k+1}\|_2^2 + \frac{L^2}{2(r-L)} \|\vv^{k+1}-\Tilde{\vv}^{k+1}\|_2^2.
\end{eqnarray}
Combining \eqref{eq:divergence:NASGD-13} and \eqref{eq:divergnce:NASGD-14}, we have
\begin{eqnarray}
\label{eq:divergence:NASGD-15}
-\delta &\leq& \mathbb{E}\left(f(\vv^{k+1}) + \frac{r}{2}\|\vv^{k+1}-\vw^k\|_2^2 \right) - \mathbb{E}\left(f(\Tilde{\vv}^{k+1}) + \frac{r}{2}\|\Tilde{\vv}^{k+1}-\vw^k\|_2^2\right) \\ \nonumber
&+&  \frac{L^2}{2(r-L)}\mathbb{E}\|\vv^{k+1}-\Tilde{\vv}^{k+1}\|_2^2 + \frac{r-L}{2}\mathbb{E}\|\vw^k - \Tilde{\vv}^{k+1}\|_2^2.
\end{eqnarray}
By rearrangement of the above inequality \eqref{eq:divergence:NASGD-15} and using Lemma~\ref{lemma:divergence-NASGD:4}, we obtain the result.
\end{proof}

\begin{lemma}
\label{lemma:divergence-NASGD:6}
If the constants satisfy $r>L$, then we have the following bounds
\begin{equation}
\label{eq:divergence:NASGD-16}
\mathbb{E}\left(f(\vv^k)-f(\vv^{k+1})\right) \geq \frac{r}{2}\mathbb{E}\|\vw^k-\vv^{k+1}\|_2^2 + r\mathbb{E}\langle \vw^k -\vv^k, \Tilde{\vv}^{k+1}-\vw^k \rangle - \tau\delta,
\end{equation}

\begin{equation}
\label{eq:divergence:NASGD-17}
\mathbb{E}\left(f(\vv^*)-f(\vv^{k+1})\right) \geq \frac{r}{2}\mathbb{E}\|\vw^k-\vv^{k+1}\|_2^2 + r\mathbb{E}\langle \vw^k -\vv^*, \Tilde{\vv}^{k+1}-\vw^k \rangle - \tau\delta,
\end{equation}
where $\tau := \frac{L^2}{r(r-L)}$ and $\vv^*$ is the minimum .
\end{lemma}

\begin{proof}
With Lemma~\ref{lemma:divergence-NASGD:3}, we have
\begin{equation}
\label{eq:divergence:NASGD-18}
-f(\Tilde{\vv}^{k+1}) \geq -f(\vw^k) - \langle \Tilde{\vv}^{k+1}-\vw^k, \nabla f(\vw^k) \rangle - \frac{L}{2}\|\Tilde{\vv}^{k+1}-\vw^k \|_2^2.
\end{equation}
Using the convexity of $f$, we have
$$
f(\vv^k) - f(\vw^k) \geq \langle \vv^k-\vw^k, \nabla f(\vw^k)\rangle,
$$
i.e.,
\begin{equation}
\label{eq:divergence:NASGD-19}
f(\vv^k) \geq f(\vw^k) + \langle \vv^k-\vw^k, \nabla f(\vw^k)\rangle.    
\end{equation}

According to the definition of $\Tilde{\vv}^{k+1}$ in \eqref{eq:divergence-NASGD:2}, i.e.,
$$
\Tilde{\vv}^{k+1} = \min_\vv Q_r(\vv, \vw^k) = \min_\vv \langle \vv-\vw^k, \nabla f(\vw^k)\rangle + \frac{r}{2}\|\vv-\vw^k\|_2^2,
$$
and the optimization condition gives
\begin{equation}
\label{eq:divergence:NASGD-20}
\Tilde{\vv}^{k+1} = \vw^k - \frac{1}{r}\nabla f(\vw^k).    
\end{equation}

Substituting \eqref{eq:divergence:NASGD-20} into \eqref{eq:divergence:NASGD-19}, we obtain
\begin{equation}
\label{eq:divergence:NASGD-21}
f(\vv^k) \geq f(\vw^k) + \langle \vv^k-\vw^k, r(\vw^k-\Tilde{\vv}^{k+1}) \rangle.
\end{equation}

Direct summation of \eqref{eq:divergence:NASGD-18} and \eqref{eq:divergence:NASGD-21} gives
\begin{equation}
\label{eq:divergence:NASGD-22}
f(\vv^k) - f(\Tilde{\vv}^{k+1})\geq \left(r-\frac{L}{2}\right)\|\Tilde{\vv}^{k+1}-\vw^k\|_2^2 + r\langle \vw^k-\vv^k, \Tilde{\vv}^{k+1}-\vw^k \rangle.
\end{equation}
Summing \eqref{eq:divergence:NASGD-22} and \eqref{eq:divergence:NASGD-10}, we obtain the inequality \eqref{eq:divergence:NASGD-16}
\begin{equation}
\label{eq:divergence:NASGD-23}
\mathbb{E}\left[f(\vv^k) - f(\vv^{k+1})\right] \geq 
\frac{r}{2}\mathbb{E}\|\vw^k-\vv^{k+1}\|_2^2 + r\mathbb{E}\langle \vw^k -\vv^k, \Tilde{\vv}^{k+1}-\vw^k \rangle - \tau\delta.
\end{equation}

On the other hand, with the convexity of $f$, we have
\begin{equation}
\label{eq:divergence:NASGD-24}
f(\vv^*) - f(\vw^k) \geq \langle \vv^*-\vw^k, \nabla f(\vw^k) \rangle = \langle \vv^*-\vw^k, r(\vw^k-\Tilde{\vv}^{k+1})\rangle.
\end{equation}
The summation of \eqref{eq:divergence:NASGD-18} and \eqref{eq:divergence:NASGD-24} resulting in
\begin{equation}
\label{eq:divergence:NASGD-25}
f(\vv^*) - f(\Tilde{\vv}^{k+1}) \geq \left(r-\frac{L}{2}\right)\|\vw^k - \Tilde{\vv}^{k+1}\|_2^2 + r\langle \vw^k-\vv^*, \Tilde{\vv}^{k+1}-\vw^k \rangle.
\end{equation}

Summing \eqref{eq:divergence:NASGD-25} and \eqref{eq:divergence:NASGD-10}, we obtain
\begin{equation}
\label{eq:divergence:NASGD-26}
\mathbb{E}\left(f(\vv^*)-f(\vv^{k+1})\right) \geq \frac{r}{2}\mathbb{E}\|\vw^k-\vv^{k+1}\|_2^2 + r\mathbb{E}\langle \vw^k -\vv^*, \Tilde{\vv}^{k+1}-\vw^k \rangle - \tau\delta,
\end{equation}
which is the same as \eqref{eq:divergence:NASGD-17}.
\end{proof}

\begin{theorem}[Uncontrolled Bound of NASGD (Theorem~1 restate)]
Let the constant satisfies $r<L$ and the sequence $\{\vv^k\}_{k\geq0}$ be generated by NASGD, then we have
\begin{equation}
\label{eq:divergence-NASGD:27}
\mathbb{E}[f(\vv^k) -\min_\vv f(\vv)] = O(k).
\end{equation}
\end{theorem}

\begin{proof}
We denote 
$$
F^k := \mathbb{E}(f(\vv^k)-f(\vv^*)).
$$
By $\eqref{eq:divergence:NASGD-16} \times (t_k-1) + \eqref{eq:divergence:NASGD-17}$, we have
\begin{eqnarray}
\label{eq:divergence-NASGD:28}
\frac{2[(t_k-1)F^k - t_kF^{k+1}]}{r} &\geq& t_k\mathbb{E}\|\vv^{k+1}-\vw^k\|_2^2 \\ \nonumber
&+& 2\mathbb{E}\langle \Tilde{\vv}^{k+1}-\vw^k, t_k\vw^k-(t_k-1)\vv^k-\vv^* \rangle - \frac{2\tau t_k\delta}{r}.
\end{eqnarray}
With $t_{k-1}^2 = t_k^2-t_k$, $\eqref{eq:divergence-NASGD:28} \times t_k$ yields
\begin{eqnarray}
\label{eq:divergence-NASGD:29}
\frac{2[t_{k-1}^2F^k-t_k^2F^{k+1}]}{r} &\geq& \mathbb{E}\|t_k\vv^{k+1}-t_k\vw^k\|_2^2 \\ \nonumber
&+& 2t_k\mathbb{E}\langle \Tilde{\vv}^{k+1}-\vw^k, t_k\vw^k-(t_k-1)\vv^k-\vv^*\rangle - \frac{2\tau t_k^2\delta}{r}
\end{eqnarray}

Substituting $\va=t_k\vv^{k+1} - (t_k-1)\vv^k-\vv^*$ and $\vb=t_k\vw^k-(t_k-1)\vv^k-\vv^*$ into identity
\begin{equation}
\label{eq:divergence-NASGD:30}
\|\va-\vb\|_2^2 + 2\langle \va-\vb, \vb\rangle = \|\va\|_2^2 - \|\vb\|_2^2.
\end{equation}
It follows that
\begin{eqnarray}
\label{eq:divergence-NASGD-31}
&&\mathbb{E}\|t_k\vv^{k+1}-t_k\vw^k\|_2^2 + 2t_k\mathbb{E}\langle \Tilde{\vv}^{k+1}-\vw^k, t_k\vw^k-(t_k-1)\vv^k-\vv^* \rangle\\ \nonumber
&=& \mathbb{E}\|t_k\vv^{k+1}-t_k\vw^k\|_2^2 + 2t_k\mathbb{E}\langle \vv^{k+1}-\vw^k, t_k\vw^k-(t_k-1)\vv^k-\vv^* \rangle \\ \nonumber
&&+ 2t_k\mathbb{E}\langle \Tilde{\vv}^{k+1}-\vv^{k+1}, t_k\vw^k-(t_k-1)\vv^k-\vv^* \rangle\\ \nonumber
&\underbrace{=}{\eqref{eq:divergence-NASGD:30}}& \mathbb{E}\|t_k\vv^{k+1}-(t_k-1)\vv^k-\vv^*\|_2^2 - \|t_k\vw^k-(t_k-1)\vv^k-\vv^*\|_2^2 \\ \nonumber
&&+ 2t_k\mathbb{E}\langle \Tilde{\vv}^{k+1}-\vv^{k+1}, t_k\vw^k-(t_k-1)\vv^k-\vv^* \rangle\\ \nonumber
&=& \mathbb{E}\|t_k\vv^{k+1}-(t_k-1)\vv^k-\vv^*\|_2^2 - \mathbb{E}\|t_{k-1}\vv^k-(t_{k-1}-1)\vv^{k-1}-\vv^*\|_2^2  \\ \nonumber
&+&2t_k \mathbb{E}\langle \Tilde{\vv}^{k+1}-\vv^{k+1}, t_{k-1}\vv^k-(t_{k-1}-1)\vv^{k-1}-\vv^*\rangle. \\ \nonumber
\end{eqnarray}

In the third identity, we used the fact $t_k\vw^k = t_k\vv^k+(t_{k-1}-1)(\vv^k-\vv^{k-1})$. If we denote $u^k = \mathbb{E}\|t_{k-1}\vv^k-(t_{k-1}-1)\vv^{k-1}-\vv^* \|_2^2$,\ \   \eqref{eq:divergence-NASGD:29} can be rewritten as
\begin{eqnarray}
\label{eq:divergence-NASGD:32}
\frac{2t_k^2F^{k+1}}{r} + u^{k+1} &\leq& \frac{2t_{k-1}^2F^k}{r} + u^k + \frac{2\tau t_k^2\delta}{r}\\ \nonumber
&+& 2t_k\mathbb{E}\langle \vv^{k+1}-\Tilde{\vv}^{k+1}, t_{k-1}\vv^k - (t_{k-1}-1)\vv^{k-1}-\vv^* \rangle\\ \nonumber
&\leq& \frac{2t_k^2F^k}{r} + u^k + \frac{2\tau t_k^2\delta}{r} + t_{k-1}^2R^2,
\end{eqnarray}
where we used
\begin{eqnarray*}
&&2t_k\mathbb{E}\langle \vv^{k+1}-\Tilde{\vv}^{k+1}, t_{k-1}\vv^k-(t_{k-1}-1)\vv^{k-1}-\vv^*\rangle\\
&\leq& t_k^2\mathbb{E}\|\vv^{k+1}-\Tilde{\vv}^{k+1}\|_2^2 + \mathbb{E}\|t_{k-1}\vv^k-(t_{k-1}\vv^k- (t_{k-1}-1)\vv^{k-1}-\vv^* ) \|_2^2\\
&=& 2t_k^2\delta/r + t_{k-1}^2R^2.
\end{eqnarray*}
Denoting
$$
\xi_k := \frac{2t_{k-1}^2F^k}{r} + u^k,
$$
then, we have
\begin{equation}
\label{eq:divergence-NASGD:33}
\xi_{k+1} \leq \xi_0 + (\frac{2\tau \delta}{r} + R^2) \sum_{i=1}^k t_i^2 = O(k^3).
\end{equation}
With the fact, $\xi_k := \frac{2t_{k-1}^2F^k}{r} \geq \Omega(k^2)F^k$, we then proved the result.
\end{proof}

\section{NAG with $\delta$-Inexact Oracle \& Experimental Settings in Section 3.1}\label{appendix:Delta:Inexact:Gradient}
In \cite{devolder2014first}, the authors defines $\delta$-inexact gradient oracle for convex smooth optimization as follows:
\begin{definition}[$\delta$-Inexact Oracle] \label{Inexact-def} \cite{devolder2014first}
For a convex $L$-smooth function $f: \RR^d \rightarrow \RR$. For $\forall \vw\in \RR^d$ and exact first-order oracle returns a pair $(f(\vw), \nabla f(\vw)) \in \RR\times \RR^d$ so that for $\forall \vv \in \RR^d$ we have
$$
0\leq f(\vv) - \big(f(\vw) + \langle\nabla f(\vw), \vv-\vw \rangle\big) \leq \frac{L}{2}\|\vw - \vv\|_2^2.
$$
A $\delta$-inexact oracle returns a pair $\big(f^\delta(\vw), \nabla f^\delta(\vw)\big) \in \RR \times \RR^d$ so that $\forall \vv \in \RR^d$ we have
$$
0\leq f(\vv) - \big(f^\delta(\vw) + \langle\nabla f^\delta(\vw), \vv-\vw \rangle\big) \leq \frac{L}{2}\|\vw - \vv\|_2^2 + \delta.
$$
\end{definition}

We have the following convergence results of GD and NAG under a $\delta$-Inexact Oracle for convex smooth optimization.

\begin{theorem}\label{Inexact-convergence}\cite{devolder2014first}\footnote{We adopt the result from \cite{mrtz}.}
Consider 
$$
\min f(\vw),\ \ \vw \in \RR^d,
$$
where $f(\vw)$ is convex and $L$-smooth with $\vw^*$ being the minimum. Given access to $\delta$-inexact oracle, GD with step size $1/L$ returns a point $\vw^k$ after $k$ steps so that
$$
f(\vw^k) - f(\vw^*) = 
O\left(\frac{L}{k}\right) + \delta.
$$
On the other hand, NAG, with step size $1/L$ returns
$$
f(\vw^k) - f(\vw^*) = 
O\left(\frac{L}{k^2}\right) + O(k\delta).
$$
\end{theorem}

Theorem~\ref{Inexact-convergence} says that NAG may not robust to a $\delta$-inexact gradient. 
In the following, we will study the numerical behavior of a variety of first-order algorithms for convex smooth optimizations with the following different inexact gradients.

{\bf Constant Variance Gaussian Noise:}
We consider the inexact oracle where the true gradient is contaminated with a Gaussian noise $\mathcal{N}(0, 0.001^2)$. We run $50$K iterations of different algorithms. For SRNAG, we restart after every $200$ iterations. Fig.~\ref{fig:Quadratic-Compare-Optimization} (b) shows the iteration vs.\ optimal gap, $f(\vx^k) - f(\vx^*)$, with $\vx^*$ being the minimum. NAG with the inexact gradient due to constant variance noise does not converge. GD performs almost the same as ARNAG asymptotically, because ARNAG restarts too often and almost degenerates into GD. GD with constant momentum outperforms the three schemes above, and SRNAG slightly outperforms GD with constant momentum.

{\bf Decaying Variance Gaussian Noise:}
Again, consider minimizing \eqref{eq:quadratic} with the same experimental setting as before except that $\nabla f(\vx)$ is now contaminated with a decaying Gaussian noise $\mathcal{N}(0, (\frac{0.1}{\floor{t/100}+1})^2)$. For SRNAG, we restart every $200$ iterations in the first $10k$ iterations, and restart every $400$ iterations in the remaining $40$K iterations. Fig.~\ref{fig:Logistic-Regression-Comparison} (c) shows the iteration vs. optimal gap by different schemes. ARNAG still performs almost the same as GD. The path of NAG is oscillatory. GD with constant momentum again outperforms the previous three schemes. Here SRNAG significantly outperforms all the other schemes.

{\bf Logisitic Regression for MNIST Classification:}
We apply the above schemes with stochastic gradient to train a logistic regression model for MNIST classification \cite{lecun-mnisthandwrittendigit-2010}. We consider five different schemes, namely, SGD, SGD $+$ (constant) momentum, NASGD, ASGD, and SRSGD.
In ARSGD, we perform restart based on the loss value of the mini-batch training data. In SRSGD, we restart the NAG momentum after every $10$ iterations. 
We train the logistic regression model with a $\ell_2$ weight decay of $10^{-4}$ by running $20$ epochs using different schemes with batch size of $128$. The step sizes for all the schemes are set to $0.01$. Fig.~\ref{fig:Logistic-Regression-Comparison} plots the training loss vs.\ iteration. In this case, NASGD does not converge, and SGD with momentum does not speed up SGD. ARSGD's performance is on par with SGD's. Again, SRSGD gives the best performance with the smallest training loss among these five schemes.

\section{Convergence of SRSGD}\label{appendix:Convergence:SRSGD}
We prove the convergence of Nesterov accelerated SGD with scheduled restart, i.e., the convergence of SRSGD. We denote that $\theta^k := \frac{t_k-1}{t_{k+1}}$ in the Nesterov iteration and $\hat{\theta}^k$ is its use in the restart version, i.e., SRSGD. For any restart frequency $F$ (positive integer), we have $\hat{\theta}^k = \theta^{k-\floor{k/F}*F}$. In the restart version, we can see that
$$
\hat{\theta}^k \leq \theta^F =: \bar{\theta} < 1.
$$

\begin{lemma}
\label{lemma:convergence:SRSGD-7}
Let the constant satisfies $r>L$ and the sequence $\{\vv^k\}_{k\geq 0}$ be generated by the SRSGD with restart frequency $F$ (any positive integer), we have
\begin{equation}
\label{eq:convergence:SRSGD-34}
\sum_{i=1}^k \|\vv^i-\vv^{i-1}\|_2^2 \leq \frac{r^2kR^2}{(1-\bar{\theta})^2},
\end{equation}
where $\bar{\theta} := \theta^F < 1$ and 
$R:=\sup_{{\bf x}}\{\|\nabla f({\bf x})\|\}$.
\end{lemma}

\begin{proof}
It holds that
\begin{eqnarray}
\label{eq:convergence:SRSGD-35}
\|\vv^{k+1} - \vw^k\|_2 &=& \|\vv^{k+1}-\vv^k+ \vv^k - \vw^k\|_2 \\ \nonumber
&\geq& \|\vv^{k+1}-\vv^k\|_2 - \|\vv^k-\vw^k\|_2 \\ \nonumber
&\geq& \|\vv^{k+1}-\vv^k\|_2 - \bar{\theta}\|\vv^k-\vv^{k-1}\|_2.
\end{eqnarray}
Thus,
\begin{eqnarray}
\label{eq:convergence:SRSGD-36}
\|\vv^{k+1}-\vw^k\|_2^2 &\geq& \left(\|\vv^{k+1}-\vv^k\| - \bar{\theta}\|\vv^k-\vv^{k-1}\|\right)^2 \\ \nonumber
&=& \|\vv^{k+1}-\vv^k\|_2^2 - 2\bar{\theta}\|\vv^k-\vv^{k-1}\|_2\|\vv^k-\vv^{k-1}\|_2 + \bar{\theta}^2\|\vv^k-\vv^{k-1}\|_2^2 \\ \nonumber
&\geq& (1-\bar{\theta})\|\vv^{k+1}-\vv^k\|_2^2 - \bar{\theta}(1-\bar{\theta})\|\vv^{k+1}-\vv^k\|_2^2.
\end{eqnarray}
Summing \eqref{eq:convergence:SRSGD-36} from $k=1$ to $K$, we get
\begin{equation}
\label{eq:convergence:SRSGD-37}
(1-\bar{\theta})^2\sum_{k=1}^K \|\vv^k-\vv^{k-1}\|_2^2 \leq \sum_{k=1}^K \|\vv^{k+1}-\vw^k\| \leq r^2KR^2.
\end{equation}
\end{proof}

In the following, we denote
$$
\mathcal{A} := \{k\in Z^+| \mathbb{E}f(\vv^k) \geq \mathbb{E}f(\vv^{k-1})\}.
$$

\begin{theorem}[Convergence of SRSGD](Theorem~2 restate)
\label{Theorem:convergence:SRSGD-2}
For any $L$-smooth function $f$, let the constant satisfies $r > L$ and the sequence $\{\vv^k\}_{k\geq 0}$ be generated by the SRSGD with restart frequency $F$ (any positive integer). Assume that $\mathcal{A}$ is finite, then we have
\begin{equation}
\label{eq:convergence:SRSGD-38}
\min_{1\leq k\leq K} \{\mathbb{E}\|\nabla f(\vw^k)\|_2^2\} = O(\frac{1}{r}+\frac{1}{rK}).
\end{equation}
Therefore for $\forall \epsilon>0$, to get $\epsilon$ error bound, we just need to set $r=O(\frac{1}{\epsilon})$ and $K=O(\frac{1}{\epsilon^2})$.
\end{theorem}

\begin{proof}
$L$-smoothness of $f$, i.e., Lipschitz gradient continuity, gives us
\begin{equation}
\label{eq:convergence:SRSGD-39}
f(\vv^{k+1}) \leq f(\vw^k) + \langle \nabla f(\vw^k), \vv^{k+1}-\vw^k \rangle + \frac{L}{2}\|\vv^{k+1}-\vw^k\|_2^2
\end{equation}

Taking expectation, we get
\begin{equation}
\label{eq:convergence:SRSGD-40}
\mathbb{E}f(\vv^{k+1}) \leq \mathbb{E}f(\vw^k) - r\mathbb{E}\|\nabla f(\vw^k)\|_2^2 + \frac{r^2LR^2}{2}.
\end{equation}

On the other hand, we have
\begin{equation}
\label{eq:convergence:SRSGD-41}
f(\vw^k) \leq f(\vv^k) + \hat{\theta}^k\langle \nabla f(\vv^k), \vv^k-\vv^{k-1} \rangle + \frac{L(\hat{\theta}^k)^2}{2}\|\vv^k-\vv^{k-1}\|_2^2.
\end{equation}

Then, we have
\begin{eqnarray}
\label{eq:convergence:SRSGD-42}
\mathbb{E}f(\vv^{k+1}) &\leq& \mathbb{E}f(\vv^k) + \hat{\theta}^k \mathbb{E}\langle \nabla f(\vv^k), \vv^k-\vv^{k-1}\rangle  \\ \nonumber
&+& \frac{L(\hat{\theta}^k)^2}{2}\mathbb{E}\|\vv^k-\vv^{k-1}\|_2^2 - r\mathbb{E}\|\nabla f(\vw^k)\|_2^2 + \frac{r^2LR^2}{2}.
\end{eqnarray}

We also have
\begin{equation}
\label{eq:convergence:SRSGD-43}
\hat{\theta}^k\langle \nabla f(\vv^k), \vv^k-\vv^{k-1} \rangle \leq \hat{\theta}^k\left(f(\vv^k)-f(\vv^{k-1}) + \frac{L}{2}\|\vv^k-\vv^{k-1}\|_2^2 \right).
\end{equation}

We then get that

\begin{equation}
\label{eq:convergence:SRSGD-44}
\mathbb{E}f(\vv^{k+1}) \leq \mathbb{E}f(\vv^k) + \hat{\theta}^k \left( \mathbb{E}f(\vv^k) - \mathbb{E}f(\vv^{k-1}) \right)  - r\mathbb{E}\|\nabla f(\vw^k)\|_2^2 + A_k,
\end{equation}
where
$$
A_k := \mathbb{E}\frac{L}{2}\|\vv^k-\vv^{k-1}\|_2^2 + \mathbb{E}\frac{L(\hat{\theta}^k)^2}{2}\mathbb{E}\|\vv^k-\vv^{k-1}\|_2^2 + \frac{r^2LR^2}{2}.
$$
Summing the inequality gives us
\begin{eqnarray}
\label{eq:convergence:SRSGD-45}
\mathbb{E}f(\vv^{K+1}) \leq \mathbb{E}f(\vv^0) &+& \Tilde{\theta}\sum_{k\in \mathcal{A}}\left( \mathbb{E}f(\vv^k) - \mathbb{E}f(\vv^{k-1}) \right)  \\ \nonumber
&-& r\sum_{k=1}^K\mathbb{E}\|\nabla f(\vw^k)\|_2^2 + \sum_{k=1}^K A_k.
\end{eqnarray}

It is easy to see that
$$
\sum_{k\in \mathcal{A}}\left(\mathbb{E}f(\vv^k) - \mathbb{E}f(\vv^{k-1}) \right) = \bar{R}<+\infty, 
$$
for the finiteness of $\mathcal{A}$, and
$$
\sum_{k=1}^K A_k = O(r^2K).
$$
\end{proof}

\section{Datasets and Implementation Details}
\label{sec:all-datasets-implementation}
\subsection{CIFAR}
\label{sec:datasets-implementation}
The CIFAR10 and CIFAR100 datasets \cite{krizhevsky2009learning} 
consist of $50$K training images and $10$K test images from $10$ and $100$ classes, respectively. Both training and test data are color images of size $32 \times 32$. We run our CIFAR experiments on Pre-ResNet-110, 290, 470, 650, and 1001 with 5 different seeds \cite{he2016identity}. We train each model for $200$ epochs with batch size of $128$ and initial learning rate of $0.1$, which is decayed by a factor of 10 at the 80th, 120th, and 160th epoch. The weight decay rate is $5 \times 10^{-5}$ and the momentum for the SGD baseline is 0.9. Random cropping and random horizontal flipping are applied to training data. Our code is modified based on the Pytorch classification project~\cite{bearpaw@github},\footnote{Implementation available at \href{https://github.com/bearpaw/pytorch-classification}{https://github.com/bearpaw/pytorch-classification}} which was also used by Liu et al. \cite{liu2020on}. We provide the restarting frequencies for the exponential and linear scheme for CIFAR10 and CIFAR100 in Table~\ref{tab:restart-f-cifar} below. Using the same notation as in the main text, we denote $F_{i}$ as the restarting frequency at the $i$-th learning rate.

\renewcommand\arraystretch{1.3}
\setlength{\tabcolsep}{4pt}
\begin{table}[h!]
\caption{Restarting frequencies for CIFAR10 and CIFAR100 experiments} \label{tab:restart-f-cifar}
\vspace{2mm}
\centering
\fontsize{8pt}{1em}\selectfont
\begin{tabular}{c|c|c}
\hline
 & CIFAR10 & CIFAR100  \\
\hline
Linear schedule & $F_{1} = 30, F_{2} = 60, F_{3} = 90, F_{4} = 120 \,\, (r = 2)$ & $F_{1} = 50, F_{2} = 100, F_{3} = 150, F_{4} = 200 \,\,  (r = 2)$ \\
Exponential schedule & $F_{1} = 40, F_{2} = 50, F_{3} = 63, F_{4} = 78 \,\, (r = 1.25)$ & $F_{1} = 45, F_{2} = 68, F_{3} = 101, F_{4} = 152 \,\, (r = 1.50)$ \\
\hline
\end{tabular}
\end{table}

\subsection{ImageNet}
\label{sec:imagenet-details}
The ImageNet dataset contains roughly 1.28 million training color images and $50$K validation color images from 1000 classes \cite{russakovsky2015imagenet}. We run our ImageNet experiments on ResNet-50, 101, 152, and 200 with 5 different seeds. Following \cite{he2016deep, he2016identity}, we train each model for 90 epochs with a batch size of 256 and decrease the learning rate by a factor of 10 at the 30th and 60th epoch. The initial learning rate is 0.1, the momentum is 0.9, and the weight decay rate is $1 \times 10^{-5}$. Random $224 \times 224$ cropping and random horizontal flipping are applied to training data. We use the official Pytorch ResNet implementation \cite{paszke2019pytorch},\footnote{Implementation available at \href{https://github.com/pytorch/examples/tree/master/imagenet}{https://github.com/pytorch/examples/tree/master/imagenet}} and run our experiments on 8 Nvidia V100 GPUs. We report single-crop top-1 and top-5 errors of our models. In our experiments, we set $F_{1} = 40$ at the 1st learning rate,  $F_{2} = 80$ at the 2nd learning rate, and $F_{3}$ is linearly decayed from 80 to 1 at the 3rd learning rate (see Table~\ref{tab:restart-f-imagenet}).

\renewcommand\arraystretch{1.3}
\setlength{\tabcolsep}{4pt}
\begin{table}[h!]
\caption{Restarting frequencies for ImageNet experiments} \label{tab:restart-f-imagenet}
\vspace{2mm}
\centering
\fontsize{8pt}{1em}\selectfont
\begin{tabular}{c|c}
\hline
 & ImageNet  \\
\hline
Linear schedule & $F_{1} = 40, F_{2} = 80, F_{3} \text{: linearly decayed from 80 to 1 in the last 30 epochs}$ \\
\hline
\end{tabular}
\end{table}

\subsection{Training ImageNet in Fewer Number of Epochs:}
\label{sec:imagenet-few-epochs-appendix}
Table~\ref{tab:details-imagenet-short-train} contains the learning rate and restarting frequency schedule for our experiments on training ImageNet in fewer number of epochs, i.e. the reported results in Table~\ref{tab:imagenet-short-train} in the main text. Other settings are the same as in the full-training ImageNet experiments described in Section~\ref{sec:imagenet-details} above.

\renewcommand\arraystretch{1.3}
\setlength{\tabcolsep}{4pt}
\begin{table}[h!]
\caption{Learning rate and restarting frequency schedule for ImageNet short training, i.e. Table~\ref{tab:imagenet-short-train} in the main text.} \label{tab:details-imagenet-short-train}
\vspace{2mm}
\centering
\fontsize{8pt}{1em}\selectfont
\begin{tabular}{c|c}
\hline
 & ImageNet  \\
\hline
ResNet-50 & Decrease the learning rate by a factor of 10 at the 30th and 56th epoch. Train for a total of 80 epochs.  \\
 & $F_{1} = 60, F_{2} = 105, F_{3} \text{: linearly decayed from 105 to 1 in the last 24 epochs}$ \\
\hline
ResNet-101 & Decrease the learning rate by a factor of 10 at the 30th and 56th epoch. Train for a total of 80 epochs.  \\
 & $F_{1} = 40, F_{2} = 80, F_{3} \text{: linearly decayed from 80 to 1 in the last 24 epochs}$ \\
\hline
ResNet-152 & Decrease the learning rate by a factor of 10 at the 30th and 51th epoch. Train for a total of 75 epochs.  \\
 & $F_{1} = 40, F_{2} = 80, F_{3} \text{: linearly decayed from 80 to 1 in the last 24 epochs}$ \\
\hline
ResNet-200 & Decrease the learning rate by a factor of 10 at the 30th and 46th epoch. Train for a total of 60 epochs.  \\
 & $F_{1} = 40, F_{2} = 80, F_{3} \text{: linearly decayed from 80 to 1 in the last 14 epochs}$ \\
\hline
\end{tabular}
\end{table}

\paragraph{Additional Implementation Details:} Implementation details for the ablation study of error rate vs.\ reduction in epochs and the ablation study of impact of restarting frequency are provided in Section~\ref{sec:error-rate-vs-epoch-reduction-appendix} and~\ref{sec:impact-restarting-frequency-appendix} below.

\section{SRSGD vs. SGD and SGD $+$ NM on ImageNet Classification and Other Tasks}\label{appendix:ImageNet:Others}
\subsection{Comparing with SGD with Nesterov Momentum on ImageNet Classification}
In this section, we compare SRSGD with SGD with Nesterov constant momentum (SGD $+$ NM) in training ResNets for ImageNet classification. All hyper-parameters of SGD with constant Nesterov momentum used in our experiments are the same as those of SGD described in section~\ref{sec:imagenet-details}. We list the results in Table~\ref{tab:imagenet-acc-SRSGDvsSGDNM}. Again, SRSGD remarkably outperforms SGD $+$ NM in training ResNets for ImageNet classification, and as the network goes deeper the improvement becomes more significant. 
\renewcommand\arraystretch{1.3}
\setlength{\tabcolsep}{4pt}
\begin{table*}[!ht]
\caption{Single crop validation
errors (\%) on ImageNet of ResNets trained with SGD $+$ NM and SRSGD. We report the results of SRSGD with the increasing restarting frequency in the first two learning rates. In the last learning rate, the restarting frequency is linearly decreased from 70 to 1. For baseline results, we also include the reported single-crop validation errors 
\cite{he@github} (in parentheses).}\label{tab:imagenet-acc-SRSGDvsSGDNM}
\centering
\fontsize{8pt}{0.8em}\selectfont
\begin{tabular}{c|c|c|c|c|c|c|c}
\hline
Network & \# Params & \multicolumn{2}{c|}{SGD $+$ NM} & \multicolumn{2}{c|}{SRSGD} & \multicolumn{2}{c}{Improvement} \\
\hline
 &  & top-1 & top-5  & top-1 & top-5  & top-1 & top-5  \\
\hline
ResNet-50 & $25.56$M & $24.27 \pm 0.07$  & $7.17 \pm 0.07$ & \pmb{$23.85 \pm 0.09$} & \pmb{$7.10 \pm 0.09$} & $0.42$ & $0.07$ \\
ResNet-101 & $44.55$M & $22.32 \pm 0.05$ & $6.18 \pm 0.05$ & \pmb{$22.06 \pm 0.10$} & \pmb{$6.09 \pm 0.07$} & $0.26$ & $0.09$ \\
ResNet-152 & $60.19$M & $21.77 \pm 0.14$ & $5.86 \pm 0.09$ & \pmb{$21.46 \pm 0.07$} & \pmb{$5.69 \pm 0.03$} & $0.31$ & $0.17$ \\
ResNet-200 & $64.67$M & $21.98 \pm 0.22$ & $5.99 \pm 0.20$ & \pmb{$20.93 \pm 0.13$} & \pmb{$5.57 \pm 0.05$} & $1.05$ & $0.42$ \\
\hline
\end{tabular}
\end{table*}

\subsection{Long Short-Term Memory (LSTM) Training for Pixel-by-Pixel MNIST}
In this task, we examine the advantage of SRSGD over SGD and SGD with Nesterov Momentum in training recurrent neural networks. In our experiments, we use an LSTM with different numbers of hidden units (128, 256, and 512) to classify samples from the well-known MNIST dataset~\cite{lecun@mnist}. We follow the implementation of~\cite{le2015simple} and feed each pixel of the image into the RNN sequentially. In addition, we choose a random permutation of $28 \times 28 = 784$ elements at the beginning of the experiment. This fixed permutation is applied to training and testing sequences. This task is known as permuted MNIST classification, which has become standard to measure the performance of RNNs and their ability to capture long term dependencies. 

{\bf \noindent Implementation and Training Details:} For the LSTM model, we initialize the forget bias to 1 and other biases to 0. All weights matrices are initialized orthogonally except for the hidden-to-hidden weight matrices, which are initialized to be identity matrices. We train each model for 350 epochs with the initial learning rate of 0.01. The learning rate was reduced by a factor of 10 at epoch 200 and 300. The momentum is set to 0.9 for SGD with standard and Nesterov constant momentum. The restart schedule for SRSGD is set to 90, 30, 90 . The restart schedule changes at epoch 200 and 300. In all experiments, we use batch size 128 and the gradients are clipped so that their L2 norm are at most 1. Our code is based on the code from the exponential RNN's Github.\footnote{Implementation available at \href{https://github.com/Lezcano/expRNN}{https://github.com/Lezcano/expRNN}} 

{\bf \noindent Results:} Our experiments corroborate the superiority of SRSGD over the two baselines. SRSGD yields much smaller test error and converges faster than SGD with standard and Nesterov constant momentum across all settings with different number of LSTM hidden units. We summarize our results in Table~\ref{tab:lstm-acc-SRSGDvsSGD} and Figure~\ref{fig:train-loss-vs-iters-lstm-pmnist}.  

\renewcommand\arraystretch{1.3}
\setlength{\tabcolsep}{4pt}
\begin{table*}[!ht]
\caption{Test
errors (\%) on Permuted MNIST of trained with SGD, SGD $+$ NM and SRSGD. The LSTM model has 128 hidden units. In all experiments, we use the initial learning rate of 0.01, which is reduced by a factor of 10 at epoch 200 and 300. All models are trained for 350 epochs. The momentum for SGD and SGD $+$ NM is set to 0.9. The restart schedule in SRSGD is set to 90, 30, and 90.}\label{tab:lstm-acc-SRSGDvsSGD}
\centering
\fontsize{8pt}{0.8em}\selectfont
\begin{tabular}{c|c|c|c|c|c}
\hline
Network & No. Hidden Units & SGD & SGD $+$ NM & SRSGD & Improvement over SGD/SGD $+$ NM  \\
\hline
LSTM & 128 & $10.10 \pm 0.57$  & $9.75 \pm 0.69$ & \pmb{$8.61 \pm 0.30$} & $1.49$/$1.14$ \\
LSTM & 256 & $10.42 \pm 0.63$  & $10.09 \pm 0.61$ & \pmb{$9.03 \pm 0.23$} & $1.39$/$1.06$ \\
LSTM & 512 & $10.04 \pm 0.35$  & $9.55 \pm 1.09$ & \pmb{$8.49 \pm 1.59$} & $1.55$/$1.06$ \\
\hline
\end{tabular}
\end{table*}
\begin{figure}[ht!]
\centering
\includegraphics[width=0.7\linewidth]{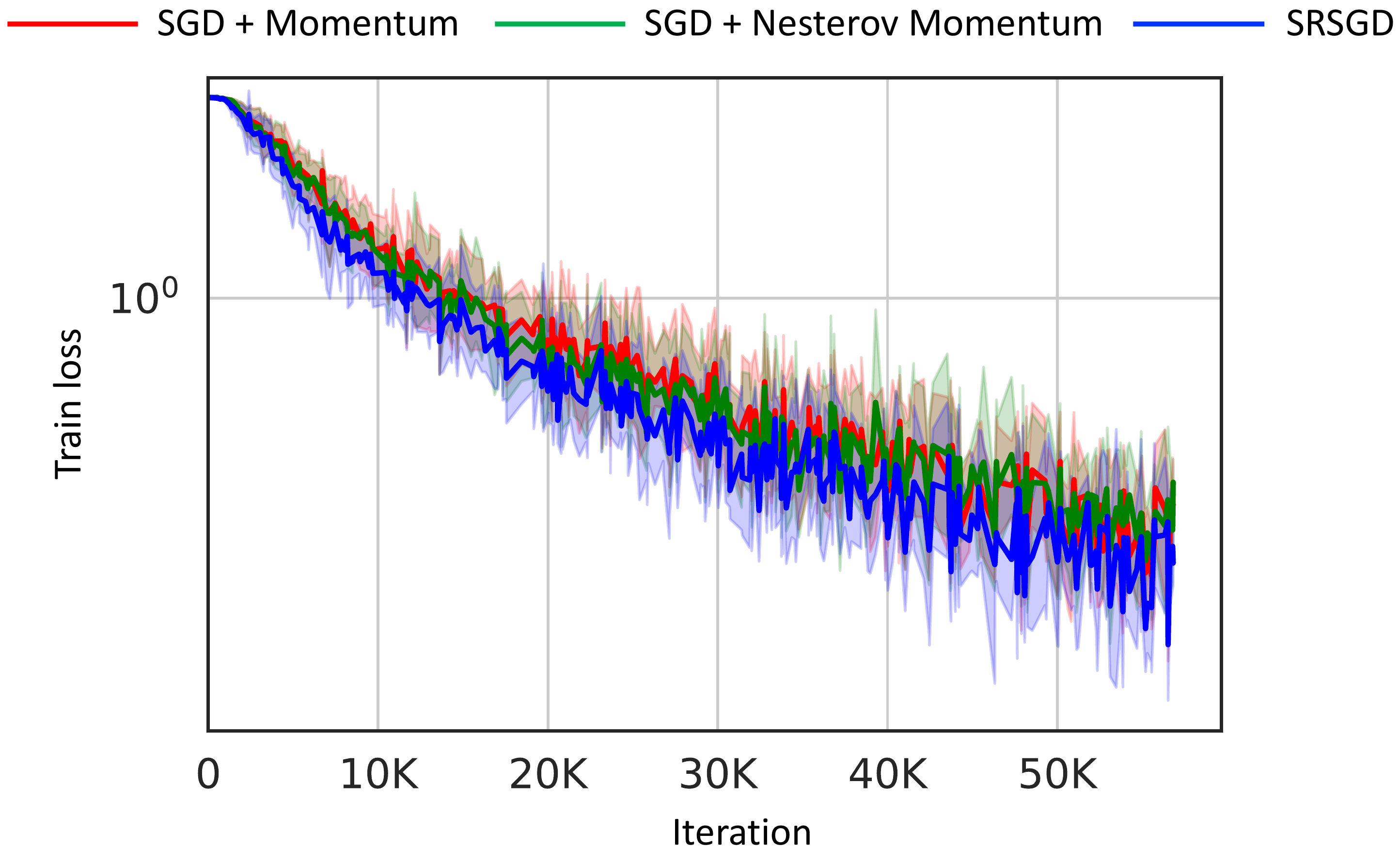}
\caption{Training loss vs. training iterations of LSTM trained with SGD (red), SGD + NM (green), and SRSGD (blue)  for PMNIST classification tasks.}
\label{fig:train-loss-vs-iters-lstm-pmnist}
\end{figure}

\subsection{Wasserstein Generative Adversarial Networks (WGAN) Training on MNIST}
We investigate the advantage of SRSGD over SGD with standard and Nesterov momentum in training deep generative models. In our experiments, we train a WGAN with gradient penalty~\cite{gulrajani2017improved} on MNIST. We evaluate our models using the discriminator's loss, i.e. the Earth Moving distance estimate, since in WGAN lower discriminator loss and better sample quality are correlated~\cite{pmlr-v70-arjovsky17a}.

{\bf \noindent Implementation and Training Details:} The detailed implementations of our generator and discriminator are given below. For the generator, we set latent\_dim to 100 and d to 32. For the discriminator, we set d to 32. We train each model for 350 epochs with the initial learning rate of 0.01. The learning rate was reduced by a factor of 10 at epoch 200 and 300. The momentum is set to 0.9 for SGD with standard and Nesterov constant momentum. The restart schedule for SRSGD is set to 60, 120, 180. The restart schedule changes at epoch 200 and 300. In all experiments, we use batch size 64. Our code is based on the code from the Pytorch WGAN-GP Github.\footnote{Implementation available at \href{https://github.com/arturml/pytorch-wgan-gp}{https://github.com/arturml/pytorch-wgan-gp}} 

\begin{lstlisting}[language=python]
import torch
import torch.nn as nn

class Generator(nn.Module):
    def __init__(self, latent_dim, d=32):
        super().__init__()
        self.net = nn.Sequential(
            nn.ConvTranspose2d(latent_dim, d * 8, 4, 1, 0),
            nn.BatchNorm2d(d * 8),
            nn.ReLU(True),

            nn.ConvTranspose2d(d * 8, d * 4, 4, 2, 1),
            nn.BatchNorm2d(d * 4),
            nn.ReLU(True),

            nn.ConvTranspose2d(d * 4, d * 2, 4, 2, 1),
            nn.BatchNorm2d(d * 2),
            nn.ReLU(True),

            nn.ConvTranspose2d(d * 2, 1, 4, 2, 1),
            nn.Tanh()
        )
    def forward(self, x):
        return self.net(x)

class Discriminator(nn.Module):
    def __init__(self, d=32):
        super().__init__()
        self.net = nn.Sequential(
            nn.Conv2d(1, d, 4, 2, 1),
            nn.InstanceNorm2d(d),
            nn.LeakyReLU(0.2),

            nn.Conv2d(d, d * 2, 4, 2, 1),
            nn.InstanceNorm2d(d * 2),
            nn.LeakyReLU(0.2),

            nn.Conv2d(d * 2, d * 4, 4, 2, 1),
            nn.InstanceNorm2d(d * 4),
            nn.LeakyReLU(0.2),

            nn.Conv2d(d * 4, 1, 4, 1, 0),
        )
    def forward(self, x):
        outputs = self.net(x)
        return outputs.squeeze()
\end{lstlisting}

{\bf \noindent Results:} Our SRSGD is still better than both the baselines. SRSGD achieves smaller discriminator loss, i.e. Earth Moving distance estimate, and converges faster than SGD with standard and Nesterov constant momentum. We summarize our results in Table~\ref{tab:gan-dloss-SRSGDvsSGD} and Figure~\ref{fig:em-estimate-vs-epoch-gan-mnist}.
We also demonstrate the digits generated by the trained WGAN in Figure~\ref{fig:mnist-gan-samples}. By visually evaluation, we observe that samples generated by the WGAN trained with SRSGD look slightly better than those generated by the WGAN trained with SGD with standard and Nesterov constant momentum.
\renewcommand\arraystretch{1.3}
\setlength{\tabcolsep}{4pt}
\begin{table*}[!ht]
\caption{Discriminator loss (i.e. Earth Moving distance estimate) of the WGAN with gradient penalty trained on MNIST with SGD, SGD $+$ NM and SRSGD. In all experiments, we use the initial learning rate of 0.01, which is reduced by a factor of 10 at epoch 200 and 300. All models are trained for 350 epochs. The momentum for SGD and SGD $+$ NM is set to 0.9. The restart schedule in SRSGD is set to 60, 120, and 180.}\label{tab:gan-dloss-SRSGDvsSGD}
\centering
\fontsize{8pt}{0.8em}\selectfont
\begin{tabular}{c|c|c|c|c}
\hline
Task & SGD & SGD $+$ NM & SRSGD & Improvement over SGD/SGD $+$ NM  \\
\hline
MNIST & $0.71 \pm 0.10$  & $0.58 \pm 0.03$ & \pmb{$0.44 \pm 0.06$} & $0.27$/$0.14$ \\
\hline
\end{tabular}
\end{table*}

\begin{figure}[h!]
\centering
\includegraphics[width=0.8\linewidth]{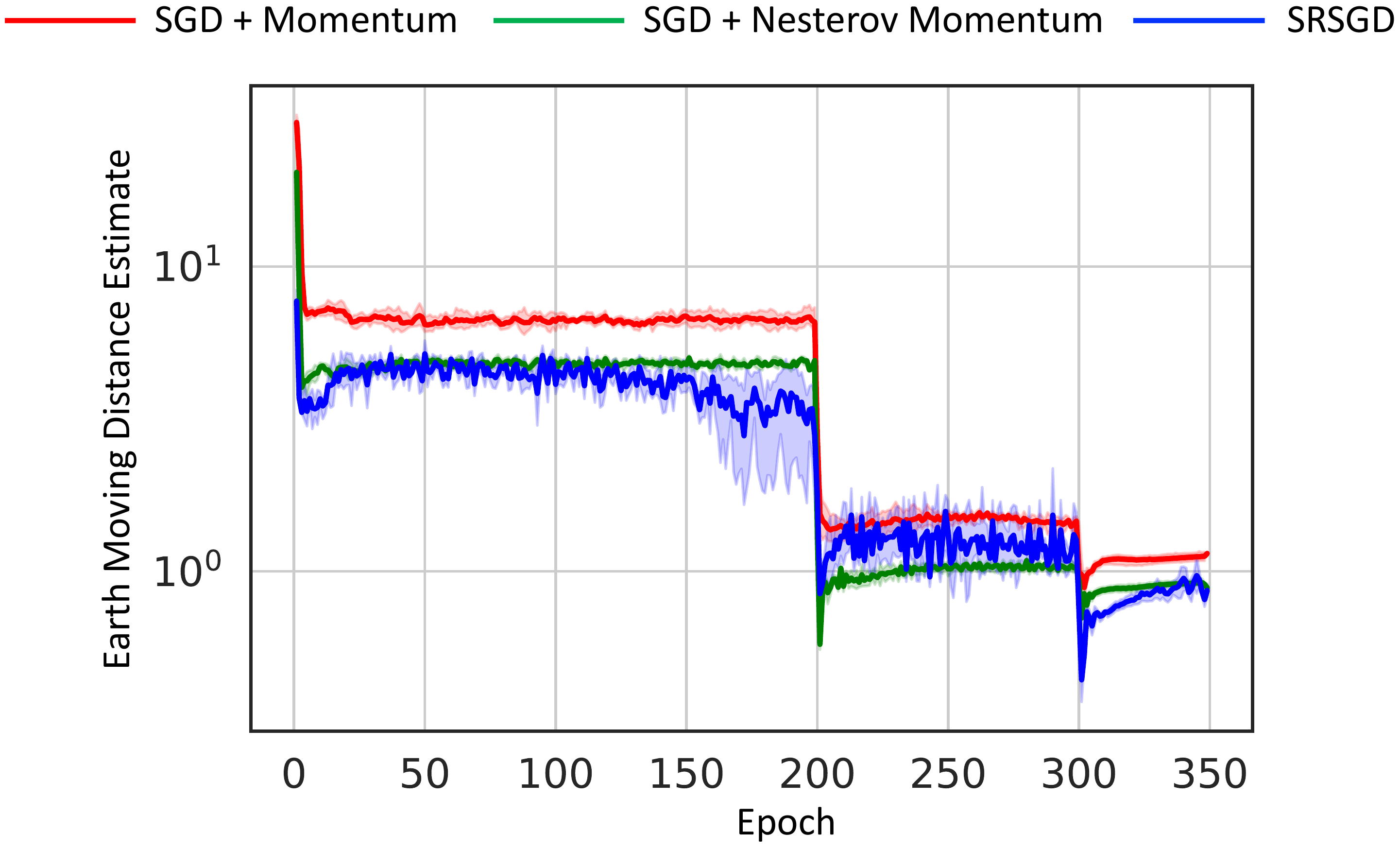}
\caption{Earth Moving distance estimate (i.e. discriminator loss) vs. training epochs of WGAN with gradient penalty trained with SGD (red), SGD + NM (green), and SRSGD (blue)  on MNIST.}
\label{fig:em-estimate-vs-epoch-gan-mnist}
\end{figure}

\begin{figure}[h!]
\centering
\includegraphics[width=1.0\linewidth]{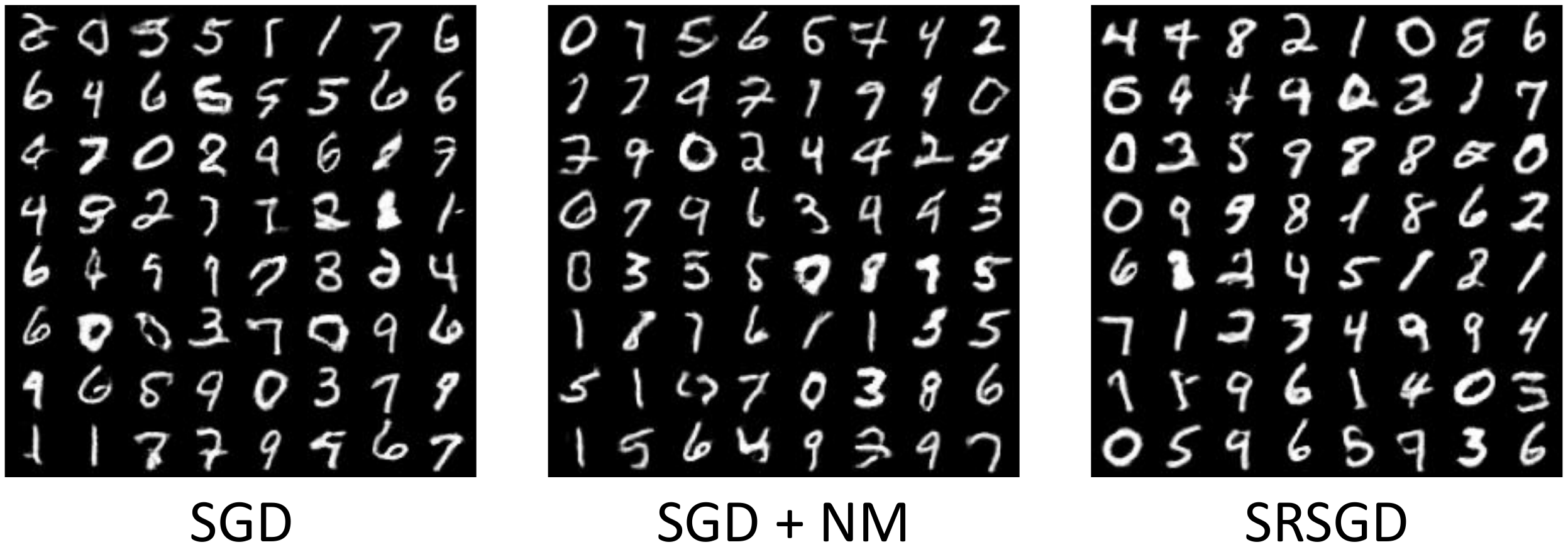}
\caption{MNIST digits generated by WGAN trained with gradient penalty by SGD (left), SGD + NM (middle), and SRSGD (right).}
\label{fig:mnist-gan-samples}
\end{figure}

\section{Error Rate vs.\ Reduction in Training Epochs}
\label{sec:error-rate-vs-epoch-reduction-appendix}
\subsection{Implementation Details}
{\bf CIFAR10 (Figure~\ref{fig:error-vs-epoch-reduction}, left, in the main text) and CIFAR100 (Figure~\ref{fig:cifar100-error-vs-epoch-reduction} in this Appendix):} Except for learning rate schedule, we use the same setting described in Section~\ref{sec:datasets-implementation} above and Section 4.1 in the main text. Table~\ref{tab:cifar-ablation-error-epoch-reduction-details} contains the learning rate schedule for each number of epoch reduction in Figure~\ref{fig:cifar100-error-vs-epoch-reduction} (left) in the main text and Figure~\ref{fig:cifar100-error-vs-epoch-reduction} below. 
\renewcommand\arraystretch{1.3}
\setlength{\tabcolsep}{4pt}
\begin{table}[h!]
\caption{Learning rate (LR) schedule for the ablation study of error rate vs.\ reduction in training epochs for CIFAR10 experiments, i.e. Figure~\ref{fig:cifar100-error-vs-epoch-reduction} (left) in the main text and for CIFAR100 experiments, i.e. Figure~\ref{fig:cifar100-error-vs-epoch-reduction} in this Appendix.} \label{tab:cifar-ablation-error-epoch-reduction-details}
\vspace{2mm}
\centering
\fontsize{8pt}{1em}\selectfont
\begin{tabular}{c|c}
\hline
\#of Epoch Reduction & LR Schedule  \\
\hline
0 & Decrease the LR by a factor of 10 at the 80th, 120th and 160th epoch. Train for a total of 200 epochs.\\
\hline
15 & Decrease the LR by a factor of 10 at the 80th, 115th and 150th epoch. Train for a total of 185 epochs.\\
\hline
30 & Decrease the LR by a factor of 10 at the 80th, 110th and 140th epoch. Train for a total of 170 epochs.\\
\hline
45 & Decrease the LR by a factor of 10 at the 80th, 105th and 130th epoch. Train for a total of 155 epochs.\\
\hline
60 & Decrease the LR by a factor of 10 at the 80th, 100th and 120th epoch. Train for a total of 140 epochs.\\
\hline
75 & Decrease the LR by a factor of 10 at the 80th, 95th and 110th epoch. Train for a total of 125 epochs.\\
\hline
90 & Decrease the LR by a factor of 10 at the 80th, 90th and 100th epoch. Train for a total of 110 epochs.\\
\hline
\end{tabular}
\end{table}

{\bf ImageNet (Figure~\ref{fig:cifar100-error-vs-epoch-reduction}, right, in the main text):} Except for the total number of training epochs, other settings are similar to experiments for training ImageNet in fewer number of epochs described in Section~\ref{sec:imagenet-few-epochs-appendix}. In particular, the learning rate and restarting frequency schedule still follow those in Table~\ref{tab:details-imagenet-short-train} above. We examine different numbers of training epochs: 90 (0 epoch reduction), 80 (10 epochs reduction), 75 (15 epochs reduction), 70 (20 epochs reduction), 65 (25 epochs reduction),  and 60 (30 epochs reduction).

\subsection{Additional Experimental Results}
Table~\ref{tab:cifar10-error-rate-vs-num-epochs} and Table~\ref{tab:imagenet-error-rate-vs-num-epochs} provide detailed test errors vs.\ number of training epoch reduction reported in Figure~\ref{fig:cifar100-error-vs-epoch-reduction} in the main text. We also conduct an additional ablation study of error rate vs.\ reduction in epochs for CIFAR100 and include the results in Figure~\ref{fig:cifar100-error-vs-epoch-reduction} and Table~\ref{tab:cifar100-error-rate-vs-num-epochs} below.

\renewcommand\arraystretch{1.3}
\setlength{\tabcolsep}{4pt}
\begin{table*}[h!]
\caption{Test error vs. number of training epochs for CIFAR10}\label{tab:cifar10-error-rate-vs-num-epochs}
\vspace{2mm}
\centering
\fontsize{8pt}{1em}\selectfont
\begin{tabular}{c|c|c|c|c|c|c|c}
\hline
Network & 110 (90 less)  & 125 (75 less) & 140 (60 less) & 155 (45 less) & 170 (30 less) & 185 (15 less) & 200 (full trainings) \\
\hline
Pre-ResNet-110 & $5.37 \pm 0.11$ & $5.27 \pm 0.17$ & $5.15 \pm 0.09$ & $5.09 \pm 0.14$ & $4.96 \pm 0.14$ & $4.96 \pm 0.13$ & \pmb{$4.93 \pm 0.13$} \\
Pre-ResNet-290 & $4.80 \pm 0.14$ & $4.71 \pm 0.13$ & $4.58 \pm 0.11$ & $4.45 \pm 0.09$ & $4.43 \pm 0.09$ & $4.44 \pm 0.11$ & \pmb{$4.37 \pm 0.15$} \\
Pre-ResNet-470 & $4.52 \pm 0.16$ & $4.43 \pm 0.12$ & $4.29 \pm 0.11$ &  $4.33 \pm 0.07$ & $4.23 \pm 0.12$ & \pmb{$4.18 \pm 0.09$} & \pmb{$4.18 \pm 0.09$} \\
Pre-ResNet-650 & $4.35 \pm 0.10$ & $4.24 \pm 0.05$ & $4.22 \pm 0.15$ & $4.10 \pm 0.15$ & $4.12 \pm 0.14$ & $4.02 \pm 0.05$ & \pmb{$4.00 \pm 0.07$} \\
Pre-ResNet-1001 & $4.23 \pm 0.19$ & $4.13 \pm 0.12$ &  $4.08 \pm 0.15$ & $4.10 \pm 0.09$ & $3.93 \pm 0.11$ & $4.06 \pm 0.14$ & \pmb{$3.87 \pm 0.07$} \\
\hline
\end{tabular}
\end{table*}

\renewcommand\arraystretch{1.3}
\setlength{\tabcolsep}{4pt}
\begin{table*}[h!]
\caption{Top 1 single crop validation error vs. number of training epochs for ImageNet}\label{tab:imagenet-error-rate-vs-num-epochs}
\vspace{2mm}
\centering
\fontsize{8pt}{1em}\selectfont
\begin{tabular}{c|c|c|c|c|c|c}
\hline
Network & 60 (30 less)  & 65 (25 less) & 70 (20 less) & 75 (15 less) & 80 (10 less) & 90 (full trainings) \\
\hline
ResNet-50 & $25.42 \pm 0.42$ & $25.02 \pm 0.15$ & $24.77 \pm 0.14$ & $24.38 \pm 0.01$ & $24.30 \pm 0.21$ & \pmb{$23.85 \pm 0.09$} \\
ResNet-101 & $23.11 \pm 0.10$ & $22.79 \pm 0.01$ & $22.71 \pm 0.21$ & $22.56 \pm 0.10$ & $22.44 \pm 0.03$ & \pmb{$22.06 \pm 0.10$} \\
ResNet-152 & $22.28 \pm 0.20$ & $22.12 \pm 0.04$ & $21.97 \pm 0.04$ &  $21.79 \pm 0.07$ & $21.70 \pm 0.07$ & \pmb{$21.46 \pm 0.07$} \\
ResNet-200 & $21.92 \pm 0.17$ & $21.69 \pm 0.20$ & $21.64 \pm 0.03$ & $21.45 \pm 0.06$ & $21.30 \pm 0.03$ & \pmb{$20.93 \pm 0.13$} \\
\hline
\end{tabular}
\end{table*}

\renewcommand\arraystretch{1.3}
\setlength{\tabcolsep}{4pt}
\begin{table*}[h!]
\caption{Test error vs. number of training epochs for CIFAR100}\label{tab:cifar100-error-rate-vs-num-epochs}
\vspace{2mm}
\centering
\fontsize{8pt}{1em}\selectfont
\begin{tabular}{c|c|c|c|c|c|c|c}
\hline
Network & 110 (90 less)  & 125 (75 less) & 140 (60 less) & 155 (45 less) & 170 (30 less) & 185 (15 less) & 200 (full trainings) \\
\hline
Pre-ResNet-110 & $24.06 \pm 0.26$ & $23.82 \pm 0.24$ & $23.82 \pm 0.28$ & $23.58 \pm 0.18$ & $23.69 \pm 0.21$ & $23.73 \pm 0.34$ & \pmb{$23.49 \pm 0.23$} \\
Pre-ResNet-290 & $21.96 \pm 0.45$ & $21.77 \pm 0.21$ & $21.67 \pm 0.37$ & $21.56 \pm 0.33$ & \pmb{$21.38 \pm 0.44$} & $21.47 \pm 0.32$ & $21.49 \pm 0.27$ \\
Pre-ResNet-470 & $21.35 \pm 0.17$ & $21.25 \pm 0.17$ & $21.21 \pm 0.18$ &  $21.09 \pm 0.28$ & $20.87 \pm 0.28$ & $20.81 \pm 0.32$ & \pmb{$20.71 \pm 0.32$} \\
Pre-ResNet-650 & $21.18 \pm 0.27$ & $20.95 \pm 0.13$ & $20.77 \pm 0.31$ & $20.61 \pm 0.19$ & $20.57 \pm 0.13$ & $20.47 \pm 0.07$ & \pmb{$20.36 \pm 0.25$} \\
Pre-ResNet-1001 & $20.27 \pm 0.17$ & $20.03 \pm 0.13$ &  $20.05 \pm 0.22$ & $19.74 \pm 0.18$ & $19.71 \pm 0.22$ & \pmb{$19.67 \pm 0.22$} & $19.75 \pm 0.11$ \\
\hline
\end{tabular}
\end{table*}

\begin{figure}[!h]
\centering
\includegraphics[width=0.5\linewidth]{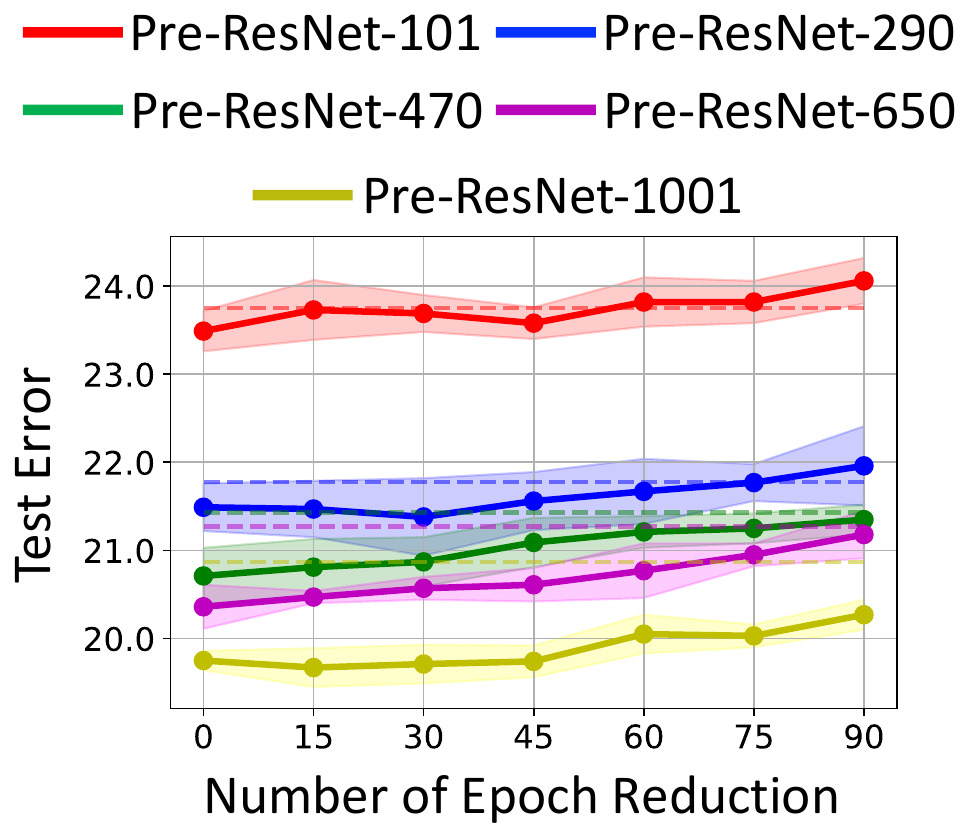}
\caption{Test error vs.\ number of epoch reduction in CIFAR100 training. The dashed lines are test errors of the SGD baseline. For CIFAR100, SRSGD training with fewer epochs can achieve comparable results to SRSGD training with full 200 epochs. In some cases, such as with Pre-ResNet-290 and 1001, SRSGD training with fewer epochs achieves even better results than SRSGD training with full 200 epochs.}
\label{fig:cifar100-error-vs-epoch-reduction}
\end{figure}

\section{Impact of Restarting Frequency for ImageNet and CIFAR100}
\label{sec:impact-restarting-frequency-appendix}
\subsection{Implementation Details}
For the CIFAR10 experiments on Pre-ResNet-290 in Figure~\ref{fig:ablation_restarting_frequency} in the main text, as well as the CIFAR100 and ImageNet experiments in Figure~\ref{fig:ablation_restarting_frequency_cifar100} and ~\ref{fig:ablation_restarting_frequency_imagenet} in this Appendix, we vary the initial restarting frequency $F_{1}$. Other settings are the same as described in Section~\ref{sec:all-datasets-implementation} above.

\subsection{Additional Experimental Results}
To complete our study on the impact of restarting frequency in Section 5.2 in the main text, we examine the case of CIFAR100 and ImageNet in this section. We summarize our results in Figure~\ref{fig:ablation_restarting_frequency_cifar100} and ~\ref{fig:ablation_restarting_frequency_imagenet} below.

\begin{figure}[!h]
\centering
\includegraphics[width=0.9\linewidth]{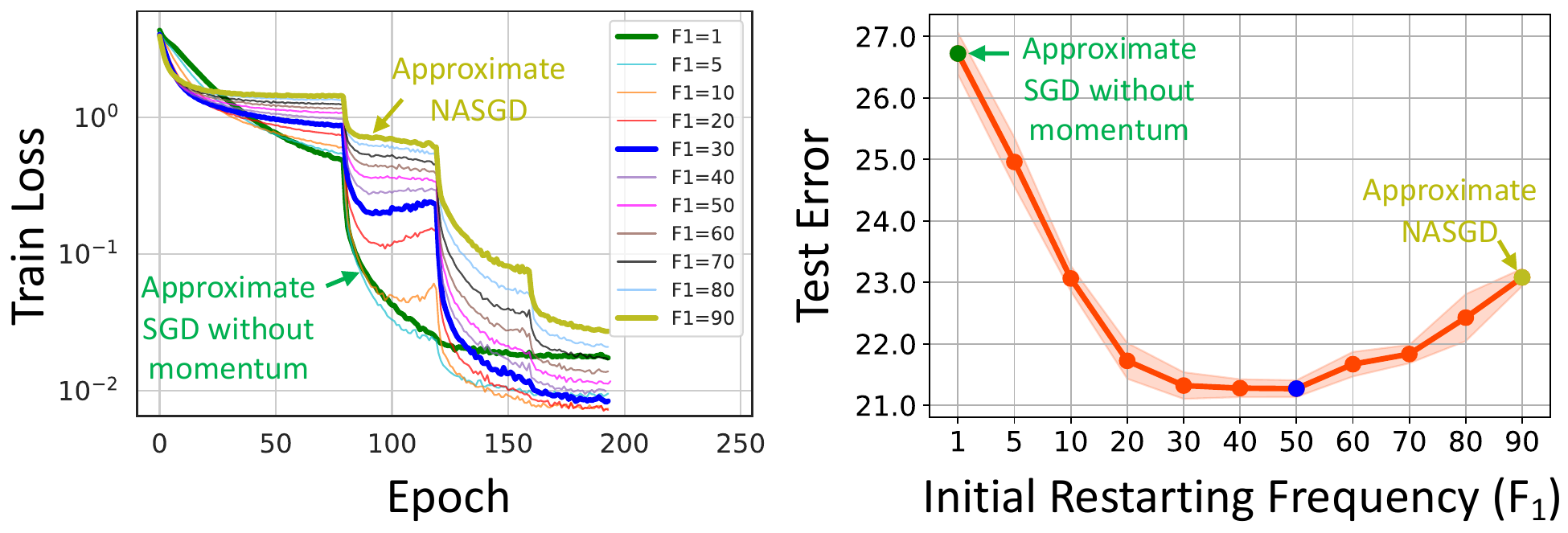}
\caption{Training loss and test error of Pre-ResNet-290 trained on CIFAR100 with different initial restarting frequencies $F_1$ (linear schedule). SRSGD with small $F_1$ approximates SGD without momentum, while SRSGD with large $F_1$ approximates NASGD.}
\label{fig:ablation_restarting_frequency_cifar100}
\end{figure}

\begin{figure}[!h]
\centering
\includegraphics[width=0.9\linewidth]{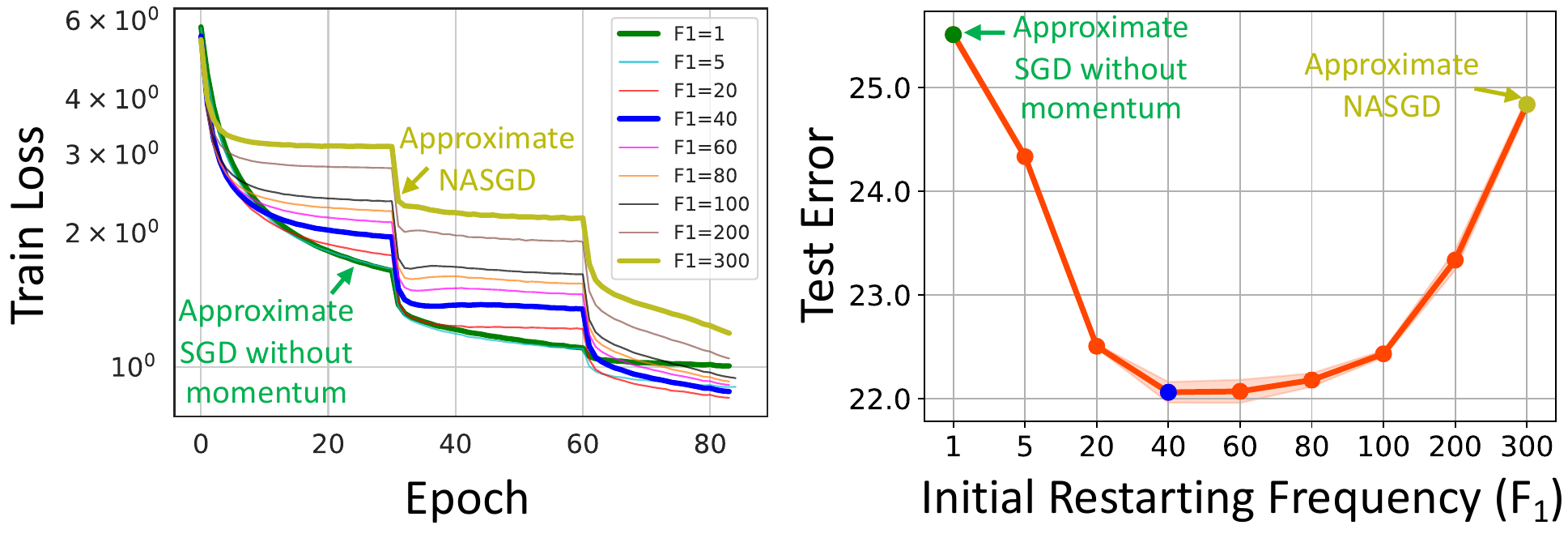}
\caption{Training loss and test error of ResNet-101 trained on ImageNet with different initial restarting frequencies $F_1$. We use linear schedule and linearly decrease the restarting frequency to 1 at the last learning rate. SRSGD with small $F_1$ approximates SGD without momentum, while SRSGD with large $F_1$ approximates NASGD.}
\label{fig:ablation_restarting_frequency_imagenet}
\end{figure}

\section{Full Training with Less Epochs at the Intermediate Learning Rates}\label{appendix:FullTrainingLessEpochs}
We explore SRSGD full training (200 epochs on CIFAR and 90 epochs on ImageNet) with less number of epochs at the intermediate learning rates and report the results in Table~\ref{tab:cifar10-error-rate-vs-reduction-still-full},~\ref{tab:cifar100-error-rate-vs-reduction-still-full},~\ref{tab:imagenet-error-rate-vs-num-epochs-full-train} and Figure~\ref{fig:cifar10-error-rate-vs-reduction-still-full},~\ref{fig:cifar100-error-rate-vs-reduction-still-full},~\ref{fig:imagenet-error-rate-vs-reduction-still-full} below. The settings and implementation details here are similar to those in Section~\ref{sec:error-rate-vs-epoch-reduction-appendix} of this Appendix, but using all 200 epochs for CIFAR experiments and 90 epochs for ImageNet experiments.

\begin{figure}[!h]
\centering
\includegraphics[width=0.55\linewidth]{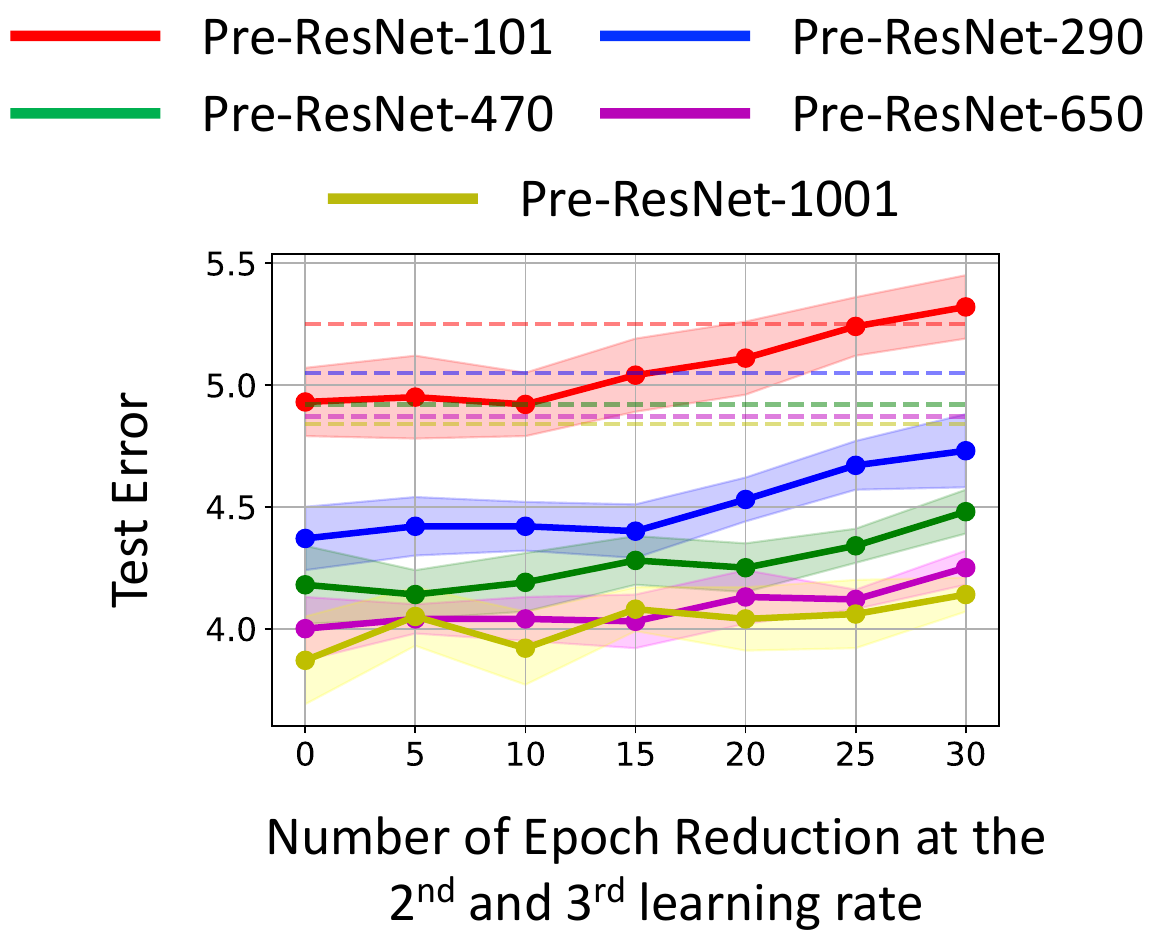}
\caption{Test error when using new learning rate schedules with less training epochs at the 2nd and 3rd learning rate for CIFAR10. We still train in full 200 epochs in this experiment. On the x-axis, 10, for example, means we reduce the number of training epochs by 10 at each intermediate learning rate, i.e. the 2nd and 3rd learning rate. The dashed lines are test errors of the SGD baseline. }
\label{fig:cifar10-error-rate-vs-reduction-still-full}
\end{figure}

\renewcommand\arraystretch{1.3}
\setlength{\tabcolsep}{4pt}
\begin{table*}[h!]
\caption{Test error when using new learning rate schedules with less training epochs at the 2nd and 3rd learning rate for CIFAR10. We still train in full 200 epochs in this experiment. In the table, 80-90-100, for example, means we reduce the learning rate by factor of 10 at the 80th, 90th, and 100th epoch.}\label{tab:cifar10-error-rate-vs-reduction-still-full}
\centering
\fontsize{8pt}{1em}\selectfont
\begin{tabular}{c|c|c|c|c|c|c|c}
\hline
Network & 80 - 90 - 100 & 80 - 95 - 110 & 80 - 100 - 120 & 80 - 105 - 130  & 80 - 110 - 140 & 80 - 115 - 150 & 80 - 120 - 160 \\
\hline
Pre-ResNet-110 & $5.32 \pm 0.14$ & $5.24 \pm 0.17$ & $5.11 \pm 0.13$ & $5.04 \pm 0.15$ & \pmb{$4.92 \pm 0.15$} & $4.95 \pm 0.12$ & $4.93 \pm 0.13$ \\
Pre-ResNet-290 & $4.73 \pm 0.13$ & $4.67 \pm 0.12$ & $4.53 \pm 0.10$ & $4.40 \pm 0.11$ & $4.42 \pm 0.09$ & $4.42 \pm 0.10$ & \pmb{$4.37 \pm 0.15$} \\
Pre-ResNet-470 & $4.48 \pm 0.16$ & $4.34 \pm 0.10$ & $4.25 \pm 0.12$ &  $4.28 \pm 0.10$ & $4.19 \pm 0.10$ & \pmb{$4.14 \pm 0.07$} & $4.18 \pm 0.09$ \\
Pre-ResNet-650 & $4.25 \pm 0.13$ & $4.12 \pm 0.06$ & $4.13 \pm 0.09$ & $4.03 \pm 0.11$ & $4.04 \pm 0.11$ & $4.04 \pm 0.04$ & \pmb{$4.00 \pm 0.07$} \\
Pre-ResNet-1001 & $4.14 \pm 0.18$ & $4.06 \pm 0.12$ &  $4.04 \pm 0.15$ & $4.08 \pm 0.09$ & $3.92 \pm 0.13$ & $4.05 \pm 0.14$ & \pmb{$3.87 \pm 0.07$} \\
\hline
\end{tabular}
\end{table*}

\renewcommand\arraystretch{1.3}
\setlength{\tabcolsep}{4pt}
\begin{table*}[h!]
\caption{Test error when using new learning rate schedules with less training epochs at the 2nd and 3rd learning rate for CIFAR100. We still train in full 200 epochs in this experiment. In the table, 80-90-100, for example, means we reduce the learning rate by factor of 10 at the 80th, 90th, and 100th epoch.}\label{tab:cifar100-error-rate-vs-reduction-still-full}
\centering
\fontsize{8pt}{1em}\selectfont
\begin{tabular}{c|c|c|c|c|c|c|c}
\hline
Network & 80 - 90 - 100 & 80 - 95 - 110 & 80 - 100 - 120 & 80 - 105 - 130  & 80 - 110 - 140 & 80 - 115 - 150 & 80 - 120 - 160 \\
\hline
Pre-ResNet-110 & $23.65 \pm 0.14$ & $23.96 \pm 0.26$ & $23.97 \pm 0.31$ & $23.53 \pm 0.13$ & $23.57 \pm 0.36$ & $23.68 \pm 0.24$ & \pmb{$23.49 \pm 0.23$} \\
Pre-ResNet-290 & $21.94 \pm 0.44$ & $21.71 \pm 0.27$ & $21.55 \pm 0.40$ & $21.44 \pm 0.31$ & \pmb{$21.37 \pm 0.45$} & $21.47 \pm 0.32$ & $21.49 \pm 0.27$ \\
Pre-ResNet-470 & $21.29 \pm 0.11$ & $21.21 \pm 0.14$ & $21.17 \pm 0.18$ &  $20.99 \pm 0.28$ & $20.81 \pm 0.22$ & $20.80 \pm 0.31$ & \pmb{$20.71 \pm 0.32$} \\
Pre-ResNet-650 & $21.11 \pm 0.24$ & $20.91 \pm 0.17$ & $20.66 \pm 0.33$ & $20.52 \pm 0.18$ & $20.51 \pm 0.16$ & $20.43 \pm 0.10$ & \pmb{$20.36 \pm 0.25$} \\
Pre-ResNet-1001 & $20.21 \pm 0.15$ & $20.00 \pm 0.11$ &  $19.86 \pm 0.19$ & \pmb{$19.55 \pm 0.19$} & $19.69 \pm 0.21$ & $19.60 \pm 0.17$ & $19.75 \pm 0.11$ \\
\hline
\end{tabular}
\end{table*}

\begin{figure}[!h]
\centering
\includegraphics[width=0.55\linewidth]{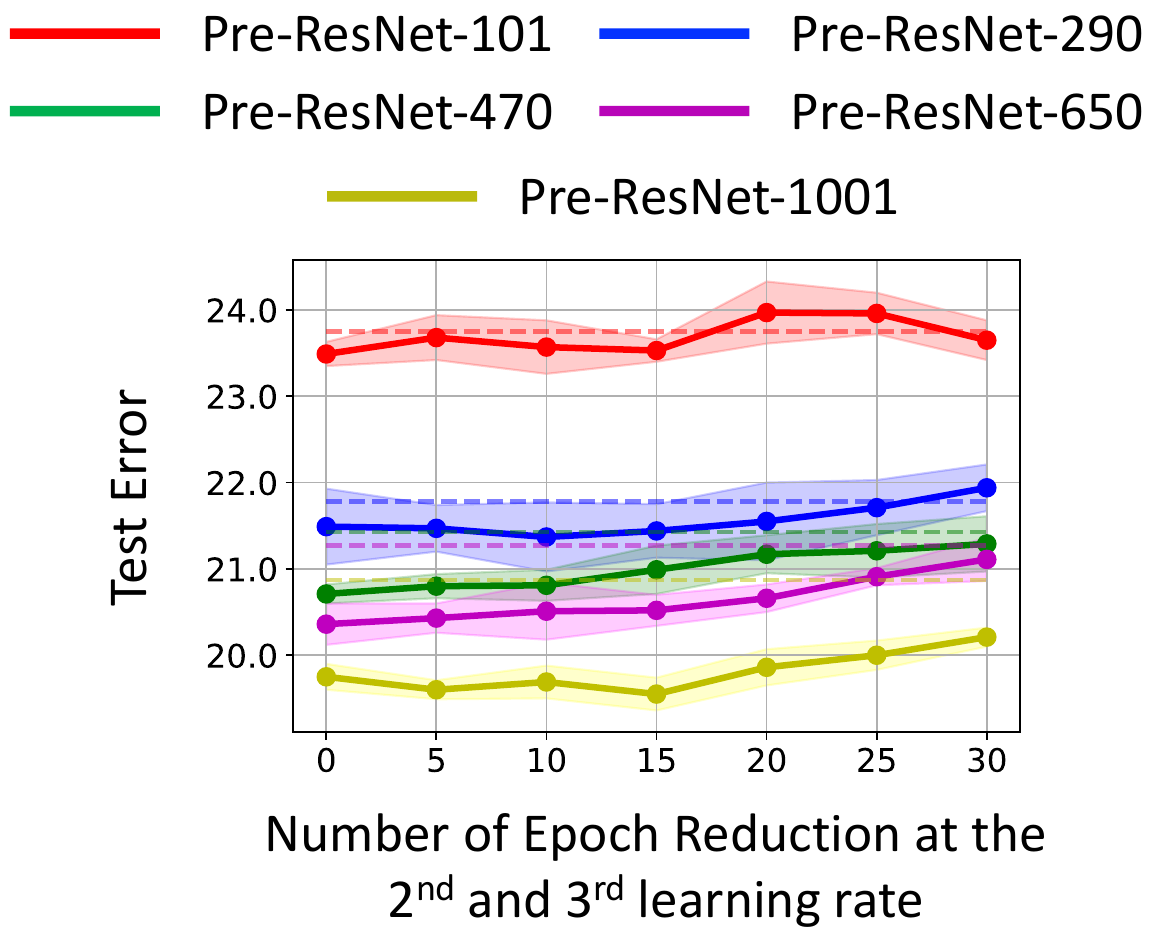}
\caption{Test error when using new learning rate schedules with less training epochs at the 2nd and 3rd learning rate for CIFAR100. We still train in full 200 epochs in this experiment. On the x-axis, 10, for example, means we reduce the number of training epochs by 10 at each intermediate learning rate, i.e. the 2nd and 3rd learning rate. The dashed lines are test errors of the SGD baseline. }
\label{fig:cifar100-error-rate-vs-reduction-still-full}
\end{figure}
\renewcommand\arraystretch{1.3}
\setlength{\tabcolsep}{4pt}
\begin{table*}[h!]
\caption{Top 1 single crop validation error when using new learning rate schedules with less training epochs at the 2nd learning rate for ImageNet. We still train in full 90 epochs in this experiment. In the table, 30-40, for example, means we reduce the learning rate by factor of 10 at the 30th and 40th epoch.}\label{tab:imagenet-error-rate-vs-num-epochs-full-train}
\vspace{2mm}
\centering
\fontsize{8pt}{1em}\selectfont
\begin{tabular}{c|c|c|c|c|c}
\hline
Network & 30 - 40  & 30 - 45 & 30 - 50 & 30 - 55 & 30 - 60 \\
\hline
ResNet-50 & $24.44 \pm 0.16$ & $24.06 \pm 0.15$ & $24.05 \pm 0.09$ & $23.89 \pm 0.14$ & \pmb{$23.85 \pm 0.09$} \\
ResNet-101 & $22.49 \pm 0.09$ & $22.51 \pm 0.05$ & $22.24 \pm 0.01$ & $22.20 \pm 0.01$ & \pmb{$22.06 \pm 0.10$} \\
ResNet-152 & $22.02 \pm 0.01$ & $21.84 \pm 0.03$ &  $21.65 \pm 0.14$ & $21.55 \pm 0.06$ & \pmb{$21.46 \pm 0.07$} \\
ResNet-200 & $21.65 \pm 0.02$ & $21.27 \pm 0.14$ & $21.12 \pm 0.02$ & $21.07 \pm 0.01$ & \pmb{$20.93 \pm 0.13$} \\
\hline
\end{tabular}
\end{table*}

\begin{figure}[!h]
\centering
\includegraphics[width=0.45\linewidth]{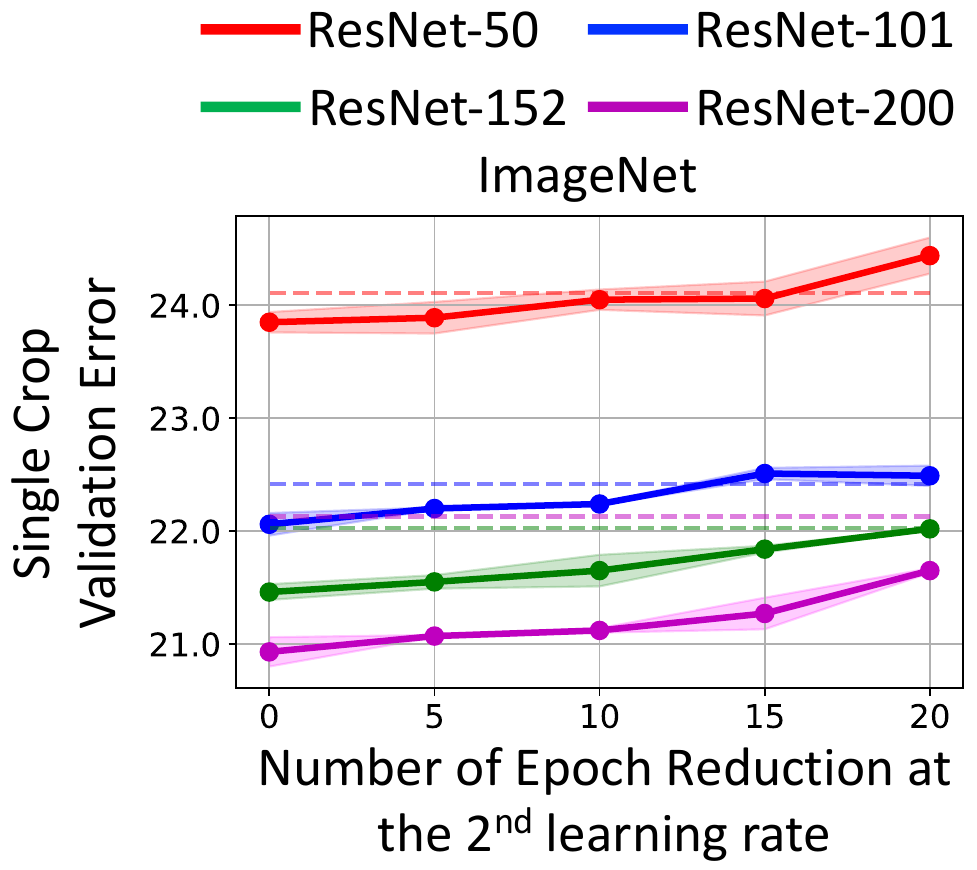}
\caption{Test error when using new learning rate schedules with less training epochs at the 2nd learning rate for ImageNet. We still train in full 90 epochs in this experiment. On the x-axis, 10, for example, means we reduce the number of training epochs by 10 at the 2nd learning rate. The dashed lines are test errors of the SGD baseline. }
\label{fig:imagenet-error-rate-vs-reduction-still-full}
\end{figure}
\newpage

\section{Visualization of SRSGD's trajectory}
Here we visualize the training trajectory through bad minima of SRSGD, SGD with constant momentum, and SGD. In particular, we train a neural net classifier on a swiss roll data as in \cite{huang2019understanding} and find bad minima along its training. Each red dot in Figure~\ref{fig:trajectory} represents the trained model after each 10 epochs in the training. From each red dot, we search for nearby bad local minima, which are the blue dots. Those bad local minima achieve good training error but bad test error. We plots the trained models and bad local minima using PCA \cite{wold1987principal} and t-SNE \cite{maaten2008visualizing} embedding. The blue color bar is for the test accuracy of bad local minima; the red color bar is for the number of training epochs.

(CONTINUED NEXT PAGE)

\begin{figure}[!h]
\centering
\includegraphics[width=1.0\linewidth]{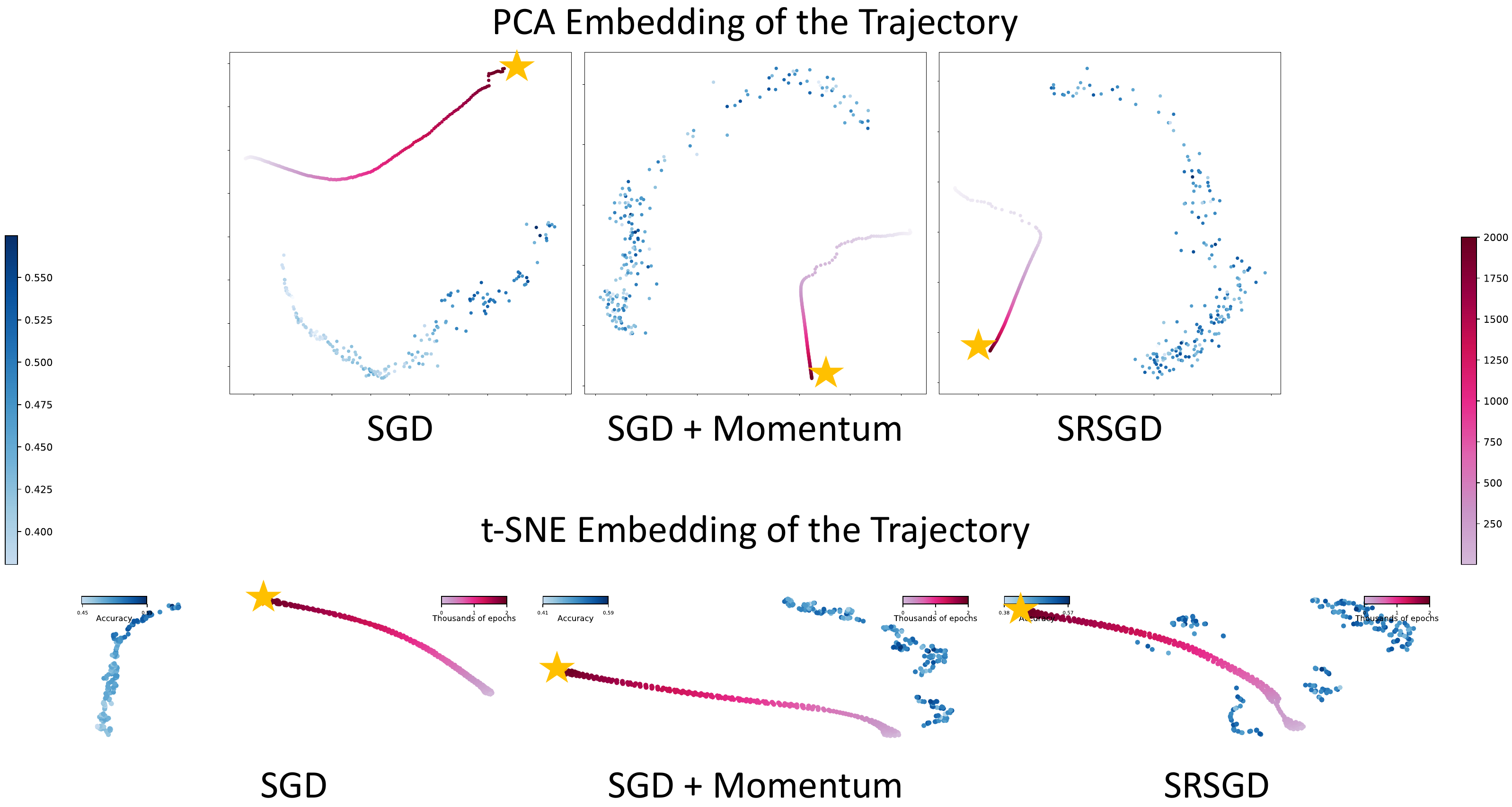}
\vspace{-0.2in}
\caption{Trajectory through bad minima of SGD, SGD with constant momentum, and SRSGD during the training: we train a neural net classifier and plot the iterates of
SGD after each ten epoch (red dots). We also plot locations of nearby “bad” minima with poor generalization
(blue dots). We visualize these using PCA and t-SNE embedding. The blue color bar is for the test accuracy of bad local minima while the red color bar is for the number of training epochs. All blue dots for SGD with constant momentum and SRSGD achieve near perfect train accuracy, but
with test accuracy below 59\%. All blue dots for SGD achieves average train accuracy of 73.11\% and
with test accuracy also below 59\%.  The final iterate (yellow star) of SGD, SGD with constant momentum, and SRSGD achieve 73.13\%, 99.25\%, and 100.0\% test accuracy, respectively.}
\label{fig:trajectory}
\end{figure}
\newpage
\section{SRSGD Implementation in Pytorch}
\label{appendix:PyTorch}
\begin{lstlisting}[language=python]
import torch
from .optimizer import Optimizer, required

class SRSGD(Optimizer):
    """
    Scheduled Restart SGD.
    Args:
        params (iterable): iterable of parameters to optimize 
                            or dicts defining parameter groups.
        lr (float): learning rate.
        weight_decay (float, optional): weight decay (L2 penalty) (default: 0)
        iter_count (integer): count the iterations mod 200
    Example:
         >>> optimizer = torch.optim.SRSGD(model.parameters(), lr=0.1,
                            weight_decay=5e-4, iter_count=1)
        >>> optimizer.zero_grad()
        >>> loss_fn(model(input), target).backward()
        >>> optimizer.step()
        >>> iter_count = optimizer.update_iter()
    Formula:
        v_{t+1} = p_t - lr*g_t
        p_{t+1} = v_{t+1} + (iter_count)/(iter_count+3)*(v_{t+1} - v_t)
    """
    def __init__(self, params, lr=required, weight_decay=0., 
                    iter_count=1, restarting_iter=100):
        if lr is not required and lr < 0.0:
            raise ValueError("Invalid learning rate: {}".format(lr))
        if weight_decay < 0.0:
            raise ValueError("Invalid weight_decay value: {}".format(weight_decay))
        if iter_count < 1:
            raise ValueError("Invalid iter count: {}".format(iter_count))
        if restarting_iter < 1:
            raise ValueError("Invalid iter total: {}".format(restarting_iter))
        
        defaults = dict(lr=lr, weight_decay=weight_decay, iter_count=iter_count, 
                            restarting_iter=restarting_iter)
        super(SRSGD, self).__init__(params, defaults)
    
    def __setstate__(self, state):
        super(SRSGD, self).__setstate__(state)
    
    def update_iter(self):
        idx = 1
        for group in self.param_groups:
            if idx == 1:
                group['iter_count'] += 1
                if group['iter_count'] >= group['restarting_iter']:
                    group['iter_count'] = 1
            idx += 1 
        return group['iter_count'], group['restarting_iter']
    
    def step(self, closure=None):
        """
        Perform a single optimization step.
        Arguments: closure (callable, optional): A closure that 
                    reevaluates the model and returns the loss.
        """
        loss = None
        if closure is not None:
            loss = closure()
        
        for group in self.param_groups:
            weight_decay = group['weight_decay']
            momentum = (group['iter_count'] - 1.)/(group['iter_count'] + 2.)
            for p in group['params']:
                if p.grad is None:
                    continue
                d_p = p.grad.data
                if weight_decay !=0:
                    d_p.add_(weight_decay, p.data)
                
                param_state = self.state[p]
                
                if 'momentum_buffer' not in param_state:
                    buf0 = param_state['momentum_buffer'] 
                         = torch.clone(p.data).detach()
                else:
                    buf0 = param_state['momentum_buffer']
                
                buf1 = p.data - group['lr']*d_p
                p.data = buf1 + momentum*(buf1 - buf0)
                param_state['momentum_buffer'] = buf1
        
        iter_count, iter_total = self.update_iter()
        
        return loss

\end{lstlisting}

\section{SRSGD Implementation in Keras}\label{appendix:Keras}
\begin{lstlisting}[language=python]
import numpy as np
import tensorflow as tf
from keras import backend as K
from keras.optimizers import Optimizer
from keras.legacy import interfaces
if K.backend() == 'tensorflow':
    import tensorflow as tf

class SRSGD(Optimizer):
    """Scheduled Restart Stochastic gradient descent optimizer.
    Includes support for Nesterov momentum
    and learning rate decay.
    # Arguments
        learning_rate: float >= 0. Learning rate.
    """

    def __init__(self, learning_rate=0.01, iter_count=1, restarting_iter=40, **kwargs):
        learning_rate = kwargs.pop('lr', learning_rate)
        self.initial_decay = kwargs.pop('decay', 0.0)
        super(SRSGD, self).__init__(**kwargs)
        with K.name_scope(self.__class__.__name__):
            self.iterations = K.variable(0, dtype='int64', name='iterations')
            self.learning_rate = K.variable(learning_rate, name='learning_rate')
            self.decay = K.variable(self.initial_decay, name='decay')
            # for srsgd
            self.iter_count = K.variable(iter_count, dtype='int64', name='iter_count')
            self.restarting_iter = K.variable(restarting_iter, dtype='int64', 
                                              name='restarting_iter')
        self.nesterov = nesterov

    @interfaces.legacy_get_updates_support
    @K.symbolic
    def get_updates(self, loss, params):
        grads = self.get_gradients(loss, params)
        self.updates = [K.update_add(self.iterations, 1)]
          
        momentum = (K.cast(self.iter_count, 
            dtype=K.dtype(self.decay)) - 1.)/(K.cast(self.iter_count, 
            dtype=K.dtype(self.decay)) + 2.)

        lr = self.learning_rate
        if self.initial_decay > 0:
            lr = lr * (1. / (1. + self.decay * K.cast(self.iterations,
                                                      K.dtype(self.decay))))
        # momentum
        shapes = [K.int_shape(p) for p in params]
        
        moments = [K.variable(value=K.get_value(p), dtype=K.dtype(self.decay), 
                   name='moment_' + str(i)) for (i, p) in enumerate(params)]
        
        self.weights = [self.iterations] + moments + [self.iter_count]
        for p, g, m in zip(params, grads, moments):
            v = p - lr * g
            new_p = v + momentum * (v - m)
            self.updates.append(K.update(m, v))

            # Apply constraints.
            if getattr(p, 'constraint', None) is not None:
                new_p = p.constraint(new_p)

            self.updates.append(K.update(p, new_p))
            
        condition = K.all(K.less(self.iter_count, self.restarting_iter))
        new_iter_count = K.switch(condition, self.iter_count + 1, 
                                  self.iter_count - self.restarting_iter + 1)
        self.updates.append(K.update(self.iter_count, new_iter_count))
        
        return self.updates

    def get_config(self):
        config = {'learning_rate': float(K.get_value(self.learning_rate)),
                  'decay': float(K.get_value(self.decay)),
                  'iter_count': int(K.get_value(self.iter_count)), 
                  'restarting_iter': int(K.get_value(self.restarting_iter))}
        base_config = super(SRSGD, self).get_config()
        return dict(list(base_config.items()) + list(config.items()))

\end{lstlisting}




\end{document}